%% file: paper.tex
\documentclass[11pt]{article}

\usepackage{etoolbox}
\newtoggle{neurips}
\togglefalse{neurips}
\newcommand{\neurips}[1]{\iftoggle{neurips}{#1}{}}
\newcommand{\arxiv}[1]{\iftoggle{neurips}{}{#1}}

\neurips{\usepackage[final,nonatbib]{neurips_2020}}

\usepackage[utf8]{inputenc} %
\usepackage[T1]{fontenc}    %
\usepackage{url}            %
\usepackage{booktabs}       %
\usepackage{amsfonts}       %
\usepackage{nicefrac}       %
\usepackage{microtype}      %
\usepackage{xspace}
\usepackage{enumitem}
\neurips{\setlist[itemize]{topsep=0pt}}

\usepackage{algorithm}
\usepackage{verbatim}
\usepackage[noend]{algpseudocode}
\usepackage{inconsolata}

\usepackage{setspace}
\usepackage{subfigure}

\neurips{\renewcommand{\paragraph}[1]{\par\textbf{#1}\hspace{5pt}}}

\neurips{\input{neurips_style}}
\arxiv{\input{arxiv_style}}
\input{dylan}
\input{noah}

\usepackage[suppress]{color-edits} %
\addauthor{df}{ForestGreen}
\newcommand{\costis}[1]{}

\usepackage[scaled=.90]{helvet}
\usepackage{xspace}

\newcommand{\minplayer}{min-player\xspace}
\newcommand{\maxplayer}{max-player\xspace}
\newcommand{\mathand}{\quad\text{and}\quad}
\newcommand{\nonconvex}{nonconvex\xspace}
\newcommand{\nonconcave}{nonconcave\xspace}

\input{macros}

\title{Independent Policy Gradient Methods \\ for Competitive Reinforcement Learning}

\date{}

\neurips{
  \author{%
Constantinos Daskalakis\\
{\small\texttt{costis@csail.mit.edu}}\\
	  \And
  	Dylan J.\ Foster\\
        {\small\texttt{dylanf@mit.edu}}\\ 
	  \And
Noah Golowich\\
{\small\texttt{nzg@mit.edu}}\\
\vspace{0.1cm}
\and
{Massachusetts Institute of Technology}
}
}
\arxiv{
\author{%
Constantinos Daskalakis\\
{\small\texttt{costis@csail.mit.edu}}\\
	  \and
  	Dylan J.\ Foster\\
        {\small\texttt{dylanf@mit.edu}}\\ 
	  \and
Noah Golowich\thanks{Supported by a Fannie \& John Hertz Foundation Fellowship and an NSF Graduate Fellowship.}\\
{\small\texttt{nzg@mit.edu}}\\
\and
{\large Massachusetts Institute of Technology}
              }
}

\begin{document}

\maketitle

\begin{abstract}
We obtain global, non-asymptotic convergence guarantees for independent learning algorithms in competitive reinforcement learning settings
with two agents (i.e.,~zero-sum stochastic games). We consider an episodic setting where in
each episode, each player independently selects a policy and observes only
\emph{their  own} actions and rewards, along with the state. We show that
if both players run policy gradient methods in tandem, their policies
will converge to a min-max equilibrium of the game, as long as their
learning rates follow a two-timescale rule (which is necessary). %
To
the best of our knowledge, this constitutes the first finite-sample
convergence result for {independent policy gradient methods} in competitive RL; prior work has largely focused on centralized, coordinated
procedures for equilibrium computation. \dfcomment{do we want to keep last sentence in light of concurrent work? we could also switch ``independent learning'' to ``independent policy gradient methods''}\noah{updated to independent policy gradient...hopefully should be accurate now}

\end{abstract}

\section{Introduction}
\label{sec:intro}
\input{section_introduction}

\section{Preliminaries}
\label{sec:problem_setting}
\input{section_problem_setting}

\section{Main Result}
\label{sec:algorithms}

\input{section_algorithms}

\section{Discussion}
\subsection{Toward Last-Iterate Convergence for Stochastic Games}
\label{sec:last_iterate}
\input{section_last_iterate}

\neurips{\vspace{-5pt}}
\subsection{Related Work}
\arxiv{\input{section_related_work}}
\neurips{
While we have already discussed related work most closely related to
our results, we refer the reader to \pref{app:related_work} for a more
extensive survey, both from the MARL and minimax perspective.
}
\neurips{\vspace{-5pt}}
\subsection{Future Directions}
\label{sec:discussion}
\input{section_discussion}

\neurips{\newpage}

\neurips{
\section*{Broader Impact}
This is a theoretical paper, and we expect that the immediate ethical and
societal consequences of our results will be limited. However, we
believe that reinforcement learning more broadly will have
significant impact on society. There is much potential for benefits to
humanity in application domains including medicine and personalized
education. There is also much potential for harm---for example, while
reinforcement learning has great promise for self-driving cars and
robotic systems, deploying methods that are not safe and reliable in
these areas could lead to serious societal and economic consequences. We
hope that research into the foundations of reinforcement learning will lead to development of algorithms with better safety
and reliability.
}

\neurips{
  \begin{ack}
    C.D.~is supported by NSF Awards IIS-1741137, CCF-1617730 and CCF-1901292, by a Simons Investigator Award, and by the DOE PhILMs project (No. DE-AC05-76RL01830).   D.F. acknowledges the support of NSF TRIPODS grant \#1740751.
     N.G.~is supported by a Fannie \& John Hertz Foundation Fellowship and an NSF Graduate Fellowship.
  \end{ack}

}

\bibliography{refs,RL,minimax_opt}
\arxiv{\newpage}
\appendix

\neurips{
  \newpage
   \section{Related work}
\label{app:related_work}
\input{section_related_work}
}

\section{Proofs from Section \ref*{sec:algorithms}}
\label{app:algorithms}
\subsection{Additional Notation}
\input{appendix_scratch}

\input{appendix_algorithms}

\section{Proofs from Section \ref*{sec:last_iterate}}
\label{app:last_iterate}
\input{appendix_last_iterate}

\end{document}

%% file: neurips_style.tex
\usepackage[dvipsnames]{xcolor}
\usepackage[colorlinks=true, linkcolor=blue!70!black, citecolor=blue!70!black,urlcolor=black,breaklinks=true]{hyperref}
\usepackage{microtype}

\usepackage{algorithm}

\usepackage[square,numbers]{natbib}
\bibliographystyle{plainnat}

\usepackage{amsthm}
\usepackage{mathtools}
\usepackage{amsmath}
\usepackage{bbm}
\usepackage{amsfonts}
\usepackage{amssymb}

\usepackage{MnSymbol} %

\usepackage{xpatch}

\theoremstyle{definition}  %
\newtheorem{exercise}{Exercise}
\newtheorem{claim}{Claim}
\newtheorem{lemma}{Lemma}
\newtheorem{conjecture}{Conjecture}
\newtheorem{corollary}{Corollary}
\newtheorem{proposition}{Proposition}
\newtheorem{fact}{Fact}
\newtheorem{assumption}{Assumption}
\newtheorem{problem}{Problem}
\newtheorem{question}{Question}
\newtheorem{model}{Model}
\theoremstyle{plain}
\newtheorem{remark}{Remark}
\newtheorem{example}{Example}
\newtheorem{theorem}{Theorem}
\newtheorem{definition}{Definition}
\newtheorem{openproblem}{Open Problem}

\xpatchcmd{\proof}{\itshape}{\normalfont\proofnameformat}{}{}
\newcommand{\proofnameformat}{\bfseries}

\makeatletter
\newcommand{\neutralize}[1]{\expandafter\let\csname c@#1\endcsname\count@}
\makeatother
\newenvironment{thmmod}[2]
  {\renewcommand{\thetheorem}{\ref*{#1}#2}%
   \neutralize{theorem}\phantomsection
   \begin{theorem}}
   {\end{theorem}}

   \newenvironment{lemmod}[2]
  {\renewcommand{\thelemma}{\ref*{#1}#2}%
   \neutralize{lemma}\phantomsection
   \begin{lemma}}
  {\end{lemma}}

\usepackage{prettyref}
\newcommand{\pref}[1]{\prettyref{#1}}
\newcommand{\pfref}[1]{Proof of \prettyref{#1}}
\newcommand{\savehyperref}[2]{\texorpdfstring{\hyperref[#1]{#2}}{#2}}
\newrefformat{eq}{\savehyperref{#1}{\textup{(\ref*{#1})}}}
\newrefformat{eqn}{\savehyperref{#1}{Equation~\ref*{#1}}}
\newrefformat{lem}{\savehyperref{#1}{Lemma~\ref*{#1}}}
\newrefformat{def}{\savehyperref{#1}{Definition~\ref*{#1}}}
\newrefformat{line}{\savehyperref{#1}{line~\ref*{#1}}}
\newrefformat{thm}{\savehyperref{#1}{Theorem~\ref*{#1}}}
\newrefformat{corr}{\savehyperref{#1}{Corollary~\ref*{#1}}}
\newrefformat{cor}{\savehyperref{#1}{Corollary~\ref*{#1}}}
\newrefformat{sec}{\savehyperref{#1}{Section~\ref*{#1}}}
\newrefformat{app}{\savehyperref{#1}{Appendix~\ref*{#1}}}
\newrefformat{ass}{\savehyperref{#1}{Assumption~\ref*{#1}}}
\newrefformat{asm}{\savehyperref{#1}{Assumption~\ref*{#1}}}
\newrefformat{ex}{\savehyperref{#1}{Example~\ref*{#1}}}
\newrefformat{fig}{\savehyperref{#1}{Figure~\ref*{#1}}}
\newrefformat{alg}{\savehyperref{#1}{Algorithm~\ref*{#1}}}
\newrefformat{rem}{\savehyperref{#1}{Remark~\ref*{#1}}}
\newrefformat{conj}{\savehyperref{#1}{Conjecture~\ref*{#1}}}
\newrefformat{prop}{\savehyperref{#1}{Proposition~\ref*{#1}}}
\newrefformat{proto}{\savehyperref{#1}{Protocol~\ref*{#1}}}
\newrefformat{prob}{\savehyperref{#1}{Problem~\ref*{#1}}}
\newrefformat{claim}{\savehyperref{#1}{Claim~\ref*{#1}}}
\newrefformat{op}{\savehyperref{#1}{Open Problem~\ref*{#1}}}

%% file: arxiv_style.tex
\usepackage[letterpaper, left=1in, right=1in, top=1in, bottom=1in]{geometry}

\usepackage[dvipsnames]{xcolor}
\usepackage[colorlinks=true, linkcolor=blue!70!black, citecolor=blue!70!black,urlcolor=black,breaklinks=true]{hyperref}
\usepackage{microtype}

\usepackage{algorithm}

 \usepackage{natbib}
 \bibliographystyle{plainnat}
 \bibpunct{(}{)}{;}{a}{,}{,}

\usepackage{amsthm}
\usepackage{mathtools}
\usepackage{amsmath}
\usepackage{bbm}
\usepackage{amsfonts}
\usepackage{amssymb}

\usepackage{MnSymbol} %

\usepackage{xpatch}

\theoremstyle{definition}  %

\newtheorem{lemma}{Lemma}

\newtheorem{proposition}{Proposition}

\newtheorem{assumption}{Assumption}

\theoremstyle{plain}
\newtheorem{remark}{Remark}

\newtheorem{theorem}{Theorem}

\newtheorem{openproblem}{Open Problem}

\xpatchcmd{\proof}{\itshape}{\normalfont\proofnameformat}{}{}
\newcommand{\proofnameformat}{\bfseries}

\makeatletter
\newcommand{\neutralize}[1]{\expandafter\let\csname c@#1\endcsname\count@}
\makeatother
\newenvironment{thmmod}[2]
  {\renewcommand{\thetheorem}{\ref*{#1}#2}%
   \neutralize{theorem}\phantomsection
   \begin{theorem}}
   {\end{theorem}}

   \newenvironment{lemmod}[2]
  {\renewcommand{\thelemma}{\ref*{#1}#2}%
   \neutralize{lemma}\phantomsection
   \begin{lemma}}
  {\end{lemma}}

\usepackage{prettyref}
\newcommand{\pref}[1]{\prettyref{#1}}
\newcommand{\pfref}[1]{Proof of \prettyref{#1}}
\newcommand{\savehyperref}[2]{\texorpdfstring{\hyperref[#1]{#2}}{#2}}
\newrefformat{eq}{\savehyperref{#1}{\textup{(\ref*{#1})}}}
\newrefformat{eqn}{\savehyperref{#1}{Equation~\ref*{#1}}}
\newrefformat{lem}{\savehyperref{#1}{Lemma~\ref*{#1}}}
\newrefformat{def}{\savehyperref{#1}{Definition~\ref*{#1}}}
\newrefformat{line}{\savehyperref{#1}{line~\ref*{#1}}}
\newrefformat{thm}{\savehyperref{#1}{Theorem~\ref*{#1}}}
\newrefformat{corr}{\savehyperref{#1}{Corollary~\ref*{#1}}}
\newrefformat{cor}{\savehyperref{#1}{Corollary~\ref*{#1}}}
\newrefformat{sec}{\savehyperref{#1}{Section~\ref*{#1}}}
\newrefformat{app}{\savehyperref{#1}{Appendix~\ref*{#1}}}
\newrefformat{ass}{\savehyperref{#1}{Assumption~\ref*{#1}}}
\newrefformat{asm}{\savehyperref{#1}{Assumption~\ref*{#1}}}
\newrefformat{ex}{\savehyperref{#1}{Example~\ref*{#1}}}
\newrefformat{fig}{\savehyperref{#1}{Figure~\ref*{#1}}}
\newrefformat{alg}{\savehyperref{#1}{Algorithm~\ref*{#1}}}
\newrefformat{rem}{\savehyperref{#1}{Remark~\ref*{#1}}}
\newrefformat{conj}{\savehyperref{#1}{Conjecture~\ref*{#1}}}
\newrefformat{prop}{\savehyperref{#1}{Proposition~\ref*{#1}}}
\newrefformat{proto}{\savehyperref{#1}{Protocol~\ref*{#1}}}
\newrefformat{prob}{\savehyperref{#1}{Problem~\ref*{#1}}}
\newrefformat{claim}{\savehyperref{#1}{Claim~\ref*{#1}}}
\newrefformat{op}{\savehyperref{#1}{Open Problem~\ref*{#1}}}

%% file: dylan.tex
\DeclarePairedDelimiter{\abs}{\lvert}{\rvert} %
\DeclarePairedDelimiter{\brk}{[}{]}
\DeclarePairedDelimiter{\crl}{\{}{\}}
\DeclarePairedDelimiter{\prn}{(}{)}
\DeclarePairedDelimiter{\nrm}{\|}{\|}
\DeclarePairedDelimiter{\tri}{\langle}{\rangle}

\DeclareMathOperator{\En}{\mathbb{E}}

\DeclareMathOperator*{\argmin}{arg\,min} %
\DeclareMathOperator*{\argmax}{arg\,max}

\newcommand{\wh}[1]{\widehat{#1}}

\def\ddefloop#1{\ifx\ddefloop#1\else\ddef{#1}\expandafter\ddefloop\fi}
\def\ddef#1{\expandafter\def\csname bb#1\endcsname{\ensuremath{\mathbb{#1}}}}
\ddefloop ABCDEFGHIJKLMNOPQRSTUVWXYZ\ddefloop
\def\ddefloop#1{\ifx\ddefloop#1\else\ddef{#1}\expandafter\ddefloop\fi}
\def\ddef#1{\expandafter\def\csname b#1\endcsname{\ensuremath{\mathbf{#1}}}}
\ddefloop ABCDEFGHIJKLMNOPQRSTUVWXYZ\ddefloop
\def\ddef#1{\expandafter\def\csname sf#1\endcsname{\ensuremath{\mathsf{#1}}}}
\ddefloop ABCDEFGHIJKLMNOPQRSTUVWXYZ\ddefloop
\def\ddef#1{\expandafter\def\csname c#1\endcsname{\ensuremath{\mathcal{#1}}}}
\ddefloop ABCDEFGHIJKLMNOPQRSTUVWXYZ\ddefloop
\def\ddef#1{\expandafter\def\csname h#1\endcsname{\ensuremath{\widehat{#1}}}}
\ddefloop ABCDEFGHIJKLMNOPQRSTUVWXYZ\ddefloop
\def\ddef#1{\expandafter\def\csname hc#1\endcsname{\ensuremath{\widehat{\mathcal{#1}}}}}
\ddefloop ABCDEFGHIJKLMNOPQRSTUVWXYZ\ddefloop
\def\ddef#1{\expandafter\def\csname t#1\endcsname{\ensuremath{\widetilde{#1}}}}
\ddefloop ABCDEFGHIJKLMNOPQRSTUVWXYZ\ddefloop
\def\ddef#1{\expandafter\def\csname tc#1\endcsname{\ensuremath{\widetilde{\mathcal{#1}}}}}
\ddefloop ABCDEFGHIJKLMNOPQRSTUVWXYZ\ddefloop

\newcommand{\ls}{\ell}

\newcommand{\eps}{\epsilon}
\newcommand{\veps}{\varepsilon}

\newcommand{\ldef}{\vcentcolon=}
\newcommand{\rdef}{=\vcentcolon}

%% file: noah.tex
\newcommand{\nc}{\newcommand}
\nc{\DMO}{\DeclareMathOperator}
\newcount\Comments  %
\nc{\p}{\mathbb{P}}
\nc{\E}{\mathbb{E}}
\nc{\ra}{\rightarrow}
\renewcommand{\t}{\top}

\nc{\noah}[1]{\ifnum\Comments=1 {\color{brown!80!black} [ng: #1]}\fi}
\nc{\modified}[1]{\ifnum\Comments=1 {\color{purple} {#1}}\else {#1}\fi}

\nc{\MG}{\mathcal{G}}

\DMO{\pr}{Pr}

\DMO{\prox}{prox}

\nc{\todo}[1]{\ifnum\Comments=1 {\color{red}  [TODO: #1]}\fi}
\nc{\old}[1]{\ifnum\Comments=1 {\color{brown}  [OLD: #1]}\fi}

\DMO{\conv}{conv}
\nc{\lng}{\langle}
\nc{\rng}{\rangle}
\nc{\grad}{\nabla}
\DMO{\Reg}{Reg}
\DMO{\Ham}{Ham}
\DMO{\Gap}{Gap}
\DMO{\GD}{GD}
\DMO{\GDA}{GDA}
\DMO{\EG}{EG}
\DMO{\OGDA}{OGDA}
\DMO{\Unif}{Unif}
\DMO{\Tr}{Tr}

\nc{\algnst}[1]{\begin{align*}#1\end{align*}}
\nc{\algn}[1]{\begin{align}#1\end{align}}
\nc{\matx}[1]{\left(\begin{matrix}#1\end{matrix}\right)}

\nc{\nuu}{\nu}

\nc{\bv}{\mathbf{v}}
\nc{\bone}{\mathbf{1}}
\nc{\bz}{\mathbf{z}}
\nc{\bw}{\mathbf{w}}
\nc{\bb}{\mathbf{b}}
\nc{\ba}{\mathbf{a}}
\nc{\bc}{\mathbf{c}}
\nc{\BR}{\mathbb R}
\nc{\BA}{\mathbb{A}}
\nc{\BP}{\mathbb{P}}
\nc{\BC}{\mathbb C}
\nc{\bx}{\mathbf{x}}
\nc{\by}{\mathbf{y}}
\nc{\sy}{y}
\nc{\sx}{x}

\nc{\MO}{\mathcal O}
\nc{\MU}{\mathcal{U}}
\nc{\ME}{\mathcal{E}}
\nc{\MN}{\mathcal{N}}
\nc{\MK}{\mathcal{K}}
\nc{\MS}{\mathcal{S}}
\nc{\MT}{\mathcal{T}}
\nc{\BF}{\mathbb F}
\nc{\BQ}{\mathbb Q}
\nc{\MX}{\mathcal{X}}
\nc{\MA}{\mathcal{A}}
\nc{\MD}{\mathcal{D}}
\nc{\MB}{\mathcal{B}}
\nc{\MZ}{\mathcal{Z}}
\nc{\MY}{\mathcal{Y}}
\nc{\BZ}{\mathbb Z}
\nc{\BN}{\mathbb N}
\nc{\ep}{\epsilon}
\nc{\BH}{\mathbb H}
\nc{\BG}{\mathbb{G}}
\nc{\D}{\Delta}
\nc{\MF}{\mathcal{F}}
\nc{\One}{\mathbbm{1}}
\nc{\bOne}{\mathbf{1}}

\nc{\SP}{\mathsf P}
\nc{\SQ}{\mathsf Q}

\nc{\DO}{\accentset{\circ}{\D}}
\nc{\mf}{\mathfrak}
\nc{\mfp}{\mathfrak{p}}
\nc{\mfq}{\mf{q}}
\nc{\Sp}{\mbox{Spec}}
\nc{\Spm}{\mbox{Specm}}
\nc{\hookuparrow}{\mathrel{\rotatebox[origin=c]{90}{$\hookrightarrow$}}}
\nc{\hookdownarrow}{\mathrel{\rotatebox[origin=c]{-90}{$\hookrightarrow$}}}
\nc{\hra}{\hookrightarrow}
\nc{\tra}{\twoheadrightarrow}
\nc{\sgn}{{\rm sgn}}
\nc{\aut}{{\rm Aut}}
\nc{\Hom}{{\rm Hom}}
\nc{\img}{{\rm Im}}
\DMO{\id}{Id}
\DMO{\supp}{supp}
\DMO{\KL}{KL}
\DMO{\BSS}{BSS}
\DMO{\BES}{BES}
\DMO{\BGS}{BGS}
\DMO{\poly}{poly}
\nc{\indep}{\perp}

\renewcommand{\t}{\top}

%% file: macros.tex
\newcommand{\ind}[1]{^{{\scriptscriptstyle(#1)}}}

\newcommand{\bigoh}{\cO}

\newcommand{\pistar}{\pi^{\star}}

\newcommand{\indic}{\mathbbm{1}}

\newcommand{\pllong}{Polyak-\L{}ojasiewicz\xspace}

\newcommand{\Vrho}{V_{\rho}}
\newcommand{\Vstar}{V^{\star}_{\rho}}
\newcommand{\pione}{\pi_1}
\newcommand{\pitwo}{\pi_2}

\newcommand{\xstar}{x^{\star}}
\newcommand{\ystar}{y^{\star}}
\newcommand{\zstar}{z^{\star}}

\newcommand{\xset}{\Delta(\cA)^{\abs*{\cS}}}
\newcommand{\yset}{\Delta(\cB)^{\abs*{\cS}}}

\newcommand{\drho}{d_{\rho}}

\nc{\gdx}{\veps_\sx}
\nc{\gdy}{\veps_\sy}
\nc{\gd}{\veps}

\newcommand{\diamx}{D_{\cX}}
\newcommand{\diamy}{D_{\cY}}

\newcommand{\Pistar}{\Pi^{\star}}

\newcommand{\proj}{\cP}

%% file: section_introduction.tex
Reinforcement learning (RL)---in which an agent must learn to maximize reward in an unknown dynamic environment---is an important frontier for artificial intelligence research, and has shown great promise in application domains ranging from robotics~\citep{kober2013reinforcement,lillicrap2015continuous,levine2016end} to games such as Atari, Go, and Starcraft~\citep{mnih2015human,silver2017mastering,vinyals2019grandmaster}. Many of the most exciting recent applications of RL are game-theoretic in nature, with multiple agents competing for shared resources or cooperating to solve a common task in stateful environments where agents' actions influence both the state and other agents' rewards~\citep{silver2017mastering,openai2018five,vinyals2019grandmaster}. Algorithms for such {\em multi-agent reinforcement learning~(MARL)} settings must be capable of accounting for other learning agents in their environment, and must choose their actions in anticipation of the behavior of these agents. Developing efficient, reliable techniques for MARL is a crucial step toward building autonomous and robust learning agents. 

While single-player (or, non-competitive RL has seen much recent theoretical activity, including development of efficient algorithms with provable, non-asymptotic guarantees~\citep{dann2015sample,azar2017minimax,jin2018q,dean2019sample,agarwal2019optimality}, provable guarantees for MARL have been comparatively sparse. 
Existing algorithms for MARL can be classified into \emph{centralized/coordinated} algorithms and
\emph{independent/decoupled} algorithms \citep{zhang2019multi}. Centralized algorithms such as self-play assume the existence of a centralized controller that joinly optimizes with respect to all agents' policies. These algorithms are typically employed in settings where the number of players and the type of interaction (competitive, cooperative, etc.) are both known a-priori. 
On the other hand, in independent reinforcement learning, agents behave myopically and optimize their own policy while treating the environment as fixed. They observe only local information, such as their own actions, rewards, and the part of the state that is available to them. As such, independent learning algorithms are generally more versatile, as they can be applied even in uncertain environments where the type of interaction and number of other agents are not known to the individual learners.

Both centralized \citep{silver2017mastering,openai2018five,vinyals2019grandmaster} and independent \citep{matignon2012independent,foerster2017stabilising}  algorithms have enjoyed practical success across different domains. However, while centralized algorithms have experienced recent theoretical development, including provable finite-sample guarantees \modified{\citep{wei2017online,bai2020provable,xie2020learning}}, theoretical guarantees for independent reinforcement learning have remained elusive. In fact, it is known that independent algorithms may fail to converge even in simple multi-agent tasks \citep{condon1990algorithms,tan1993multi,claus1998dynamics}: When agents update their policies independently, they induce distribution shift, which can break assumptions made by classical single-player algorithms. Understanding when these algorithms work, and how to stabilize their performance and tackle distribution shift, is recognized as a major challenge in multi-agent RL \citep{matignon2012independent,hernandez2017survey}.

In this paper, we focus on understanding the convergence properties of independent reinforcement learning with \emph{policy gradient methods} \citep{williams1992simple,sutton2000policy}. Policy gradient methods form the foundation for modern applications of multi-agent reinforcement learning, with state-of-the-art performance across many domains \citep{schulman2015trust,schulman2017proximal}. Policy gradient methods are especially relevant for continuous reinforcement learning and control tasks, since they readily scale to large action spaces, and are often more stable than value-based methods, particularly with function approximation \citep{konda2000actor}. %
Independent reinforcement learning with policy gradient methods is poorly understood, and attaining global convergence results is considered an important open problem \citep[Section 6]{zhang2019multi}.

We analyze the behavior of independent policy gradient methods in Shapley's \emph{stochastic game} framework~\citep{shapley_stochastic_1953}. We focus on two-player zero-sum stochastic games with discrete state and action spaces, wherein players observe the entire joint state, take simultaneous actions, and observe rewards simultaneously, with one player trying to maximize the reward and the other trying to minimize it. To capture the challenge of independent learning, we assume that each player observes the state, reward, and {\em their own} action, but not the action chosen by the other player. We assume that the dynamics and reward distribution are unknown, so that players must optimize their policies using only realized trajectories consisting of the states, rewards, and actions. For this setting, we show that---while independent policy gradient methods may not converge in general---policy gradient methods following a {\em two-timescale rule} converge to a Nash equilibrium. %
We also show that moving beyond two-timescale rules by incorporating optimization techniques from matrix games such as optimism \citep{daskalakis2017training} or extragradient updates \citep{korpelevich1976extragradient} is likely to require new analysis techniques.

At a technical level, our result is a special case of a more general theorem, which shows that (stochastic) two-timescale updates converge to Nash equilibria for a class of \nonconvex minimax problems satisfying a certain two-sided gradient dominance property. Our results here expand the class of \nonconvex minimax problems with provable algorithms beyond the scope of prior work \citep{yang2020global}, and may be of independent interest. %

%% file: section_problem_setting.tex
We investigate the behavior of independent learning in two-player zero-sum stochastic games (or, Markov games), a simple competitive reinforcement learning setting \citep{shapley_stochastic_1953,littman_markov_1994}. In these games, two players---a {\it \minplayer} and a {\it \maxplayer}---repeatedly select actions simultaneously in a shared Markov decision process in order to minimize and maximize, respectively, a given objective function. Formally, a two-player zero-sum stochastic game is specified by a tuple $\MG = (\MS, \MA, \MB, P, R, \zeta, \rho)$:
\neurips{\begin{itemize}[topsep=0pt]}
\arxiv{\begin{itemize}}
\item $\MS$ is a finite {\it state space} of size $S = |\MS|$.
\item $\MA$ and $\MB$ are finite {\it action spaces} for the min- and max-players, of sizes $A =|\MA|$ and  $B = |\MB|$.
\item $P$ is the {\it transition probability function}, for which $P(s' \mid s,  a,b)$ denotes the probability of transitioning to state $s'$ when the current state is $s$ and the players take actions $a$ and $b$. In general we will have $\zeta_{s,a,b}\ldef{}1 - \sum_{s' \in \MS} P(s' \mid s,a,b) > 0$; this quantity represents the probability that $\MG$ {\it stops} at state $s$ if actions $a,b$ are played.
\item $R : \MS \times \MA \times \MB \ra [-1,1]$ is the {\it reward function}; $R(s,a,b)$ gives the immediate reward when the players take actions $a,b$ in state $s$. The \minplayer seeks to minimize $R$ and the \maxplayer seeks to maximize it.\footnote{We consider deterministic rewards for simplicity, but our results immediately extend to stochastic rewards.}
\item $\zeta := \min_{s,a,b} \crl{\zeta_{s,a,b}}$ is a lower bound on the probability that the game stops at any state $s$ and choices of actions $a,b$. We assume that $\zeta > 0$ throughout this paper.
\item $\rho \in \Delta(\MS)$ is the {\it initial distribution} of the state at time $t=0$.
\end{itemize}
At each time step $t \geq 0$, both players observe a state $s_t \in \MS$, pick actions $a_t \in \MA$ and $b_t \in \MB$, receive reward $r_t\ldef{}R(s_t, a_t, b_t)$, and transition to the next state $s_{t+1} \sim P(\cdot \mid s_t, a_t, b_t)$. With probability $\zeta_{s_t, a_t, b_t}$, the game stops at time $t$; since $\zeta > 0$, the game stops eventually with probability $1$.%

A pair of (randomized) policies $\pi_1:\cS\to\Delta(\cA)$, $\pi_2:\cS\to\Delta(\cB)$ induces a distribution $\pr^{\pi_1,\pi_2}$ of trajectories $(s_t, a_t, b_t, r_t)_{0 \leq t \leq T}$, where $s_0 \sim \rho $, $a_t \sim \pi_1(\cdot \mid s_t), b_t \sim \pi_2(\cdot \mid s_t)$, $r_t = R(s_t, a_t, b_t)$, and $T$ is the last time step before the game stops (which is a random variable). The {\it value function} $V_s(\pi_1,\pi_2)$ gives the expected reward when $s_0 = s$ and the plays follow $\pi_1$ and $\pi_2$:
$$
V_s(\pi_1,\pi_2) := \E_{\pi_1,\pi_2} \left[ \sum_{t =0}^T R(s_t, a_t, b_t) \mid{} s_0 = s \right],
$$
where $\E_{\pi_1,\pi_2}\brk*{\cdot}$ denotes expectation under the trajectory distribution given induced by $\pi_1$ and $\pi_2$. We set $V_\rho(\pi_1,\pi_2) : = \E_{s \sim \rho} [V_s(x,y)]$.

\paragraph{Minimax value.} \citet{shapley_stochastic_1953} showed that stochastic games satisfy a minimax theorem: For any game $\cG$, there exists a Nash equilibrium $(\pistar_1,\pistar_2)$ such that
\begin{equation}
  \label{eq:nash}
  V_\rho(\pistar_1,\pi_2) \leq V_\rho(\pistar_1,\pistar_2) \leq V_\rho(\pi_1, \pistar_2),\quad\text{for all $\pi_1,\pi_2$,}
\end{equation}
and in particular
$\Vstar\ldef{}\min_{\pi_1} \max_{\pi_2} V_\rho(\pi_1,\pi_2)= \max_{\pi_2} \min_{\pi_1} V_\rho(\pi_1,\pi_2)$. Our goal in this setting is to develop algorithms to find $\veps$-approximate Nash equilibria, i.e. to find $\pi_1$ such that
\begin{equation}
  \label{eq:minimax_goal}
  \max_{\pi_2}\Vrho(\pi_1,\pi_2) \leq \Vrho(\pistar_1,\pistar_2) + \veps,
\end{equation}
and likewise for the \maxplayer.

\paragraph{Visitation distributions.} For policies $\pi_1,\pi_2$ and an initial state $s_0$, define the {\it discounted state visitation distribution} $d_{s_0}^{\pi_1,\pi_2} \in \Delta(\MS)$ by
$$
d_{s_0}^{\pi_1,\pi_2}(s) \propto \sum_{t \geq 0} \pr^{\pi_1,\pi_2}(s_t = s | s_0),
$$
where $\pr^{\pi_1,\pi_2}(s_t = s | s_0)$ is the probability that the
game has not stopped at time $t$ and the $t$th state is $s$, given
that we start at $s_0$. We define $d_\rho^{\pi_1,\pi_2}(s) := \E_{s_0
  \sim \rho}[d_{s_0}^{\pi_1,\pi_2}(s)]$. 

\paragraph{Additional notation.}
For a vector $x\in\bbR^{d}$, we let $\nrm{x}$ denote the Euclidean norm. For a finite set $\cX$, $\Delta(\cX)$ denotes the set of all distributions over $\cX$. We adopt non-asymptotic big-oh notation: For functions
	$f,g:\cX\to\bbR_{+}$, we write $f=\bigoh(g)$ if there exists a universal constant
	$C>0$ that does not depend on problem parameters, such that $f(x)\leq{}Cg(x)$ for all $x\in\cX$. %

\section{Independent Learning}
\label{sec:independent}

\paragraph{Independent learning protocol.}
We analyze independent reinforcement learning algorithms for stochastic games in an episodic setting in which both players repeatedly execute arbitrary policies for a fixed number of episodes with the goal of producing an (approximate) Nash equilibrium.%

We formalize the notion of independent RL via the following protocol: At each episode $i$, the \minplayer proposes a policy $\pione\ind{i}:\cS\to\Delta(\cA)$ and the \maxplayer proposes a policy $\pitwo\ind{i}:\cS\to\Delta(\cB)$ independently. These policies are executed in the game $\cG$ to sample a trajectory. The \minplayer observes only its own trajectory $(s_1\ind{i}, a_1\ind{i}, r_1\ind{i}), \ldots, (s_T\ind{i},a_T\ind{i},r_T\ind{i})$, and the \maxplayer likewise observes $(s_1\ind{i}, b_1\ind{i}, r_1\ind{i})$, $\ldots,$ $(s_T\ind{i},b_T\ind{i},r_T\ind{i})$. Importantly, each player is oblivious to the actions selected by the other.

We call a pair of  algorithms for the min- and max-players an \emph{independent distributed protocol} if (1) the players only access the game $\cG$ through the oracle model above (\emph{independent oracle}), and (2) the players can only use private storage, and are limited to storing a constant number of past trajectories and parameter vectors (\emph{limited private storage}). The restriction on limited private storage aims to rule out strategies that orchestrate the players' sequences of actions in order for them to both reconstruct a good approximation of entire game $\cG$ in their memory, then solve for equilibria locally. We note that making this constraint precise is challenging, and that similar difficulties with formalizing it arise even for two-player matrix games, as discussed in~\citet{daskalakis2011near}. In any event, the policy gradient methods analyzed in this paper satisfy these formal constraints \emph{and} are independent in the intuitive sense, \modified{with the caveat that the players need a very small amount of a-priori coordination to decide which player operates at a faster timescale when executing two-timescale updates.} \modified{Because of the necessity of two-timescale updates, our algorithm does not satisfy the requirement of \emph{strong independence}, which we define to be the setting that disallows any coordination to break symmetry so as to agree on differing ``roles'' of the players (such as differing step-sizes or exploration probabilities). As discussed further in \pref{sec:last_iterate}, we leave the question of developing provable guarantees for strongly independent algorithms of this type as an important open question.} \dfcomment{wording of this part is a little bit confusing because we don't mention strong independence prior. we should explicitly say that strong independence is a term that we are defining}\noah{fixed}

\paragraph{Our question: Convergence of independent policy gradient methods.}
Policy gradient methods are widely used in practice \citep{schulman2015trust,schulman2017proximal}, and are appealing in their simplicity: Players adopt continuous policy parameterizations $x\mapsto{}\pi_x$, and $y\mapsto\pi_y$, where $x\in\cX\subseteq\bbR^{d_1}$, $y\in\cY\subseteq\bbR^{d_2}$ are parameter vectors. Each player simply treats $V_{\rho}(x,y)\ldef{}\Vrho(\pi_x,\pi_y)$ as a continuous optimization objective, and updates their policy using an iterative method for stochastic optimization, using trajectories to form stochastic gradients for $V_{\rho}$.

For example, if both players use the ubiquitous REINFORCE gradient estimator \citep{williams1992simple}, and update their policies with stochastic gradient descent, the updates for episode $i$ take the form\footnote{For
  a convex set $\cX$, $\proj_{\cX}$ denotes euclidean projection onto
  the set.}
\begin{align}
x\ind{i+1}\gets{}\proj_{\cX}(x\ind{i}-\eta_x\wh{\grad}_x\ind{i}),\mathand{}y\ind{i+1}\gets{}\proj_{\cY}(y\ind{i}+\eta_y\wh{\grad}_y\ind{i}),
  \label{eq:pg}
\end{align}
with
\begin{align}
  \wh{\grad}_x\ind{i} \ldef{} R_T\ind{i}\sum_{t=0}^{T}\grad\log\pi_x(a_t\ind{i}\mid{}s_t\ind{i}),\mathand
  \wh{\grad}_y\ind{i} \ldef{} R_T\ind{i}\sum_{t=0}^{T}\grad\log\pi_y(b_t\ind{i}\mid{}s_t\ind{i}),\label{eq:reinforce}
\end{align}
where $R_T\ind{i}\ldef\sum_{t=0}^{T}r_t\ind{i}$, and where $x\ind{0},y\ind{0}$ are initialized arbitrarily. This protocol is independent, since each player forms their respective policy gradient using only the data from their own trajectory. This leads to our central question:
\begin{center}
  \emph{When do independent agents following policy gradient updates in a zero-sum stochastic game converge to a Nash equilibrium? }
\end{center}

We focus on an $\veps$-greedy variant of the so-called \emph{direct parameterization} where $\cX =\Delta(\cA)^{\abs*{\cS}}$, $\cY = \Delta(\cB)^{\abs*{\cS}}$, $\pi_x(a\mid{}s) = (1-\veps_x)x_{s,a} + \veps_{x}/\abs*{\cA}$, and $\pi_y(a\mid{}s) = (1-\veps_y)y_{s,b} + \veps_{y}/\abs*{\cB}$, where $\veps_x$ and $\veps_y$ are exploration parameters. This is a simple model, but we believe it captures the essential difficulty of the independent learning problem.

\paragraph{Challenges of independent learning.}
Independent learning is challenging even for \emph{simple stochastic games}, which are a special type of stochastic game in which only a single player can choose an action in each state, and where there are no rewards except in certain ``sink'' states. Here, a seminal result of \citet{condon1990algorithms}, establishes that even with oracle access to the game $\cG$ (e.g., exact $Q$-functions given the opponent's policy), many naive approaches to independent learning can cycle and fail to approach equilibria, including protocols where (1) both players perform policy iteration independently, and (2) both players compute best responses at each episode. On the positive side, \citet{condon1990algorithms} also shows that if one player performs policy iteration independently while the other computes a best response at each episode, the resulting algorithm converges, which parallels our findings.

Stochastic games also generalize two-player zero-sum matrix games. Here, even with exact gradient access, it is well-known that if players update their strategies independently using online gradient descent/ascent (GDA) with the same learning rate, the resulting dynamics may cycle, leading to poor guarantees unless the entire iterate sequence is averaged \citep{daskalakis2017training,mertikopoulos2018cycles}. To make matters worse, when one moves beyond the convex-concave setting, such iterate averaging techniques may fail altogether, as their analysis critically exploits convexity/concavity of the loss function. To give stronger guarantees---either for the last-iterate or for ``most'' elements of the iterate sequence---more sophisticated techniques based on two-timescale updates or negative momentum are required. However, existing results here rely on the machinery of convex optimization, and stochastic games---even with  direct parameterization---are nonconvex-nonconcave, leading to difficulties if one attempts to apply these techniques out of the box.

In light of these challenges, it suffices to say that we are aware of no global convergence results for independent policy gradient methods (or any other independent distributed protocol, for that matter)  in general finite state/action zero-sum stochastic games.

%% file: section_algorithms.tex
We show that independent policy gradient algorithms following the
updates in \pref{eq:pg} converge to a Nash equilibrium, so long as
their learning rates follow a \emph{two-timescale} rule. The
two-timescale rule is a simple modification of the usual
gradient-descent-ascent scheme for minimax optimization in which the \minplayer uses a much
smaller stepsize than the \maxplayer (i.e., $\eta_x\ll\eta_y$), and
hence works on a slower timescale (or vice-versa). Two-timescale rules
help to avoid limit cycles in simple minimax optimization settings \citep{heusel2017gans,lin2019gradient}, and our result shows that their benefits extend to MARL as well.

\paragraph{Assumptions.}
Before stating the result, we first introduce some technical
conditions that quantify the rate of convergence. First, it is
well-known that policy gradient methods can systematically
under-explore hard-to-reach states. Our
convergence rates depend on an appropriately-defined
\emph{distribution mismatch coefficient} which bounds the difficulty
of reaching such states, generalizing results for the single-agent
setting \citep{agarwal2019optimality}. While methods based on
sophisticated exploration (e.g., \cite{dann2015sample,jin2018q}) can avoid
dependence on mismatch parameters, our goal here---similar to prior
work in this direction
\citep{agarwal2019optimality,bhandari2019global}---is to understand the
behavior of standard methods used in practice, so we take the
dependence on such parameters as a given.

Given a stochastic game $\MG$, we define the {\it minimax mismatch} coefficient for $\MG$ by:
\begin{equation}
C_{\MG} := \max\crl*{\max_{\pi_2}\min_{\pi_1\in\Pistar_1(\pi_2)}\nrm*{\frac{d^{\pi_1,\pi_2}_{\rho}}{\rho}}_{\infty}, \max_{\pi_1}\min_{\pi_2\in\Pistar_2(\pi_1)}\nrm*{\frac{d^{\pi_1,\pi_2}_{\rho}}{\rho}}_{\infty}},
\label{eq:mismatch}
\end{equation}
where $\Pistar_1(\pi_2)$ and $\Pistar_2(\pi_1)$ each denotes the set of best
responses for the min- (resp. max-) player
when the max- (resp. min-) player plays $\pi_2$ (resp. $\pi_1$).

Compared to results for the single-agent setting, which typically
scale with $\nrm{\nicefrac{d^{\pi^{\star}}_{\rho}}{\rho}}_{\infty}$, where
$\pistar$ is an optimal policy \citep{agarwal2019optimality}, the
minimax mismatch coefficient measures the worst-case ratio for each
player, given that their adversary best-responds.
While the minimax mismatch coefficient in general is larger
than its single-agent counterpart, it is still weaker than other
notions of mismatch such as concentrability
\citep{munos2003error,chen2019information,yang2019theoretical}, which---when specialized
to the two-agent setting---require that the ratio is bounded for \emph{all} pairs of
policies. The following proposition makes this observation precise.
\begin{proposition}
  \label{prop:concentratability}
  There exists a stochastic game with five states and initial
  distribution $\rho$ such that $C_{\MG}$ is bounded, but the concentrability coefficient
  $\max_{\pi_1, \pi_2} \left\| \frac{d_\rho^{\pi_1, \pi_2}}{\rho}
  \right\|_\infty$ is infinite.
\end{proposition}

Next, to ensure the variance of the REINFORCE estimator stays bounded, we require that both players use $\veps$-greedy exploration in conjunction with the basic policy gradient updates \pref{eq:pg}.
\begin{assumption}
  \label{ass:greedy}
  Both players follow the direct parameterization with $\veps$-greedy
  exploration: Policies are parameterized as $\pi_x(a\mid{}s) = (1-\veps_x)x_{s,a} + \veps_{x}/\abs*{\cA}$ and $\pi_y(a\mid{}s) = (1-\veps_y)y_{s,b} + \veps_{y}/\abs*{\cB}$, where $\veps_x,\veps_y\in\brk{0,1}$ are the \emph{exploration parameters}.
\end{assumption}

We can now state our main result.

\begin{theorem}
  \label{thm:main}Let $\eps>0$ be given. Suppose both players follow
  the independent policy gradient scheme \pref{eq:pg} with the parameterization in \pref{ass:greedy}. If the
  learning rates satisfy $\eta_x\asymp\eps^{10.5}$ and $\eta_y \asymp\ep^6$ and the exploration parameters satisfy $\gdx \asymp \ep, \gdy\asymp\ep^2$, we are guaranteed that
  \begin{align}
    \label{eq:main}
    \textstyle    \En\brk*{\frac{1}{N}\sum_{i=1}^{N}\max_{\pi_2}V_{\rho}(\pi_{x\ind{i}},\pi_2)} - \min_{\pi_1}\max_{\pi_2}V_{\rho}(\pi_1,\pi_2) \leq{} \eps
  \end{align}
  after $N\leq{}\poly(\eps^{-1},C_{\cG},S,A,B,\zeta^{-1})$ episodes.
\end{theorem}
This represents, to our knowledge, the first finite-sample, global convergence guarantee for independent policy gradient updates in stochastic games. Some key features are as follows:

\arxiv{
  \begin{itemize}

  \item Since the learning agents only use their own trajectories to
    make decisions, and only store a single parameter vector in
    memory, the protocol is indeed independent %
    in the sense
    of \pref{sec:independent}. However, an important caveat is that since the players use different learning rates, the protocol only succeeds if the rates are coordinated in advance. %

\item   The two-timescale update rule may be thought of as a
    softened ``gradient descent vs. best response'' scheme in
    which the \minplayer updates their strategy using policy gradient
    and the \maxplayer updates their policy with a best response to
    the \minplayer (since $\eta_x\ll\eta_y$). This is why the guarantee
    is asymmetric, in that it only guarantees that the iterates of the
    \minplayer are approximate Nash equilibria.\footnote{From an {\it optimization perspective}, the oracle complexity of finding a solution so that the iterates of both the min and \maxplayer  are approximate equilibria is only twice as large as that in Theorem \ref{thm:main}, since we may apply Theorem \ref{thm:main} with the roles switched.}
    We remark that the gradient descent vs. exact best response has recently been analyzed for
    linear-quadratic games \citep{zhang2019policy}, and it is possible
    to use the machinery of our proofs to show that it succeeds in
    our finite state/action stochastic game setting as well. 

\item Eq. \pref{eq:main} shows that the iterates of the \minplayer have
    low error on average, in the sense that the expected error is smaller than $\eps$ if we select
    an iterate from the sequence uniformly at
    random. Such a guarantee goes beyond what is achieved by gradient-descent-ascent (GDA) with equal
    learning rates: Even for zero-sum matrix games, the iterates of GDA can
    reach limit cycles that remain a constant distance from the
    equilibrium, so that any individual iterate in the sequence
    will have high error \citep{mertikopoulos2018cycles}. While
    averaging the iterates takes care of this issue for matrix games,
    this technique relies critically on convexity, which is not
    present in our policy gradient setting. While our guarantees are
    stronger than GDA, we believe that giving guarantees that hold for
    individual (in particular, last) iterates rather than on average over iterates is an important open problem. This is discussed further in \pref{sec:last_iterate}.

\item We have not attempted to optimize the dependence on
    $\eps^{-1}$ or other parameters, and this can almost certainly be improved.
    \end{itemize}
    }

\neurips{
  $\bullet$ Since the learning agents only use their own trajectories to
    make decisions, and only store a single parameter vector in
    memory, the protocol is independent %
    in the sense
    of \pref{sec:independent}. However, an important caveat is that 
    since the players use different learning rates, the protocol only succeeds if this is agreed upon in advance. %

    $\bullet$   The two-timescale update rule may be thought of as a
    softened ``gradient descent vs. best response'' scheme in
    which the \minplayer updates their strategy using policy gradient
    and the \maxplayer updates their policy with a best response to
    the \minplayer (since $\eta_x\ll\eta_y$). This is why the guarantee
    is asymmetric, in that it only guarantees that the iterates of the
    \minplayer are approximate Nash equilibria.\footnote{From an {\it optimization perspective}, the oracle complexity of finding a solution so that the iterates of both the min- and max-players are approximate equilibria is only twice as large as that in Theorem \ref{thm:main}, since we may apply Theorem \ref{thm:main} with the roles switched.} %
    We remark that the gradient descent vs. exact best response has recently been analyzed for
    linear-quadratic games \citep{zhang2019policy}, and it is possible
    to use the machinery of our proofs to show that it succeeds in
    our setting \modified{of stochastic games} as well. %

  $\bullet$ Eq. \pref{eq:main} shows that the iterates of the \minplayer have
    low error on average, in the sense that the expected error is smaller than $\eps$ if we select
    an iterate from the sequence uniformly at
    random. Such a guarantee goes beyond what is achieved by GDA with equal
    learning rates: Even for zero-sum matrix games, the iterates of GDA can
    reach limit cycles that remain a constant distance from the
    equilibrium, so that any individual iterate in the sequence
    will have high error \citep{mertikopoulos2018cycles}. While
    averaging the iterates takes care of this issue for matrix games,
    this technique relies critically on convexity, which is not
    present in our policy gradient setting. While our guarantees are
    stronger than GDA, we believe that giving guarantees that hold for
    individual (in particular, last) iterates rather than on average over iterates is an important open problem, and we discuss this further in \pref{sec:last_iterate}.

    $\bullet$ We have not attempted to optimize the dependence on
    $\eps^{-1}$ or other parameters, and this can almost certainly be improved.
    }

The full proof of \pref{thm:main}---as well as explicit dependence on
problem parameters---is deferred to
\pref{app:algorithms}. In the remainder of this section we sketch the
key techniques.
\paragraph{Overview of techniques.}
Our result builds on recent advances that prove that policy gradient methods
converge in single-agent reinforcement learning (\cite{agarwal2019optimality}; see also
\cite{bhandari2019global}). These results
show that while the reward function
$V_{\rho}(\pi_x)=\En_{\pi_x}\brk{\sum_{t=1}^{T}r_t\mid{}s_0\sim\rho}$ is
not convex---even for the direct parameterization---it satisfies a
favorable \emph{gradient domination} condition whenever a
distribution mismatch coefficient is bounded. This allows one to apply
standard results for finding first-order stationary points in smooth
\nonconvex optimization out of the box to derive convergence
guarantees. We show that two-player zero-sum stochastic games satisfy
an analogous \emph{two-sided gradient dominance condition}.

\begin{lemma}
  \label{lem:gd_greedy}
    Suppose that players follow the $\veps$-greedy direct
  parameterization of \pref{ass:greedy} with parameters $\veps_x$ and
  $\veps_y$. Then for all $x\in\xset$, $y\in\yset$ we have 
  \begin{align}
      \Vrho(\pi_x, \pi_y) - \min_{\pi_1} \Vrho(\pi_1,\pi_y)
    &  \leq{}
      \min_{\pi_1 \in \Pi_1^*(\pi_y)}
\nrm*{\frac{\drho^{\pi_1,\pi_y}}{\rho}}_{\infty}\prn*{    \frac{1}{\zeta}\max_{\bar{x}\in\Delta(\cA)^{\abs*{\cS}}}\tri*{\grad_{x}\Vrho(\pi_x,\pi_y),x-\bar{x}}
    + \frac{2\veps_x}{\zeta^3}},
    \label{eq:gd}
  \end{align}
  and an analogous upper bound holds for $\max_{\pi_2} \Vrho(\pi_y,\pi_2) - \Vrho(\pi_x, \pi_y)$.
\end{lemma}
Informally, the gradient dominance condition posits that for either
player to have low regret relative to the best response to the opponent's
policy, it suffices to find a near-stationary point. In particular,
while the function $x\mapsto{}\Vrho(x,y)$ is \nonconvex, the condition
\pref{eq:gd} implies that if the \maxplayer fixes their strategy, all
local minima are global for the \minplayer.

Unfortunately, compared to the single-agent setting, we are aware of no existing black-box
minimax optimization results that can exploit this condition to achieve even
asymptotic convergence guarantees. To derive our main results,
we develop a new proof that two-timescale updates find Nash equilibria
for generic minimax problems that satisfy the two-sided GD condition.
\begin{theorem}
  \label{thm:sgda_body}
  Let $\cX$ and $\cY$ be convex sets with diameters $\diamx$ and $\diamy$. Let
  $f:\cX\times\cY\to\bbR$ be any, $\ls$-smooth, $L$-Lipschitz function for which there exist
  constants $\mu_x,\mu_y,\veps_x$, and $\veps_y$ such that for all $x\in\cX$ and $y\in\cY$,
  \begin{align}
    &\max_{\bar x \in \MX, \| \bar x - x \| \leq 1} \lng x- \bar x,
    \grad_x f(x,y) \rng \geq \mu_\sx \cdot (f(x,y) -
    \min_{x'\in\cX}f(x', y)) - \gdx,\\
    &\max_{\bar y \in \cY: \nrm*{\bar{y}-y}\leq{}1} \lng\bar{y}-y, \grad_y f(x,y) \rng \geq \mu_\sy \cdot (\max_{y'\in\cY} f(x,y') - f(x, y)) - \gdy.
  \end{align}
 Then, given stochastic gradient oracles with variance at most
  $\sigma^{2}$, two-timescale stochastic gradient descent-ascent
  (Eq. \pref{eq:sgda-x} in \pref{app:sgda}) with learning rates $\eta_\sx \asymp\eps^{8}$ and $\eta_\sy
  \asymp\ep^4$ ensures that
  \begin{equation}
\textstyle    \En\brk*{\frac{1}{N}\sum_{i=1}^{N}\max_{y\in\cY}f(x\ind{i},y)} -
    \min_{x\in\cX}\max_{y\in\cY}f(x,y)\leq\eps
    \label{eq:sgda_bound}
  \end{equation}
  within $N\leq{}\poly(\eps^{-1},D_{\cX},D_{\cY}, L,\ell,
  \mu_x^{-1},\mu_y^{-1},\sigma^2)$ episodes.
\end{theorem}
A formal statement and proof of \pref{thm:sgda_body} are given in
\pref{app:sgda}. To deduce \pref{thm:main} from this result, we simply trade off the bias due to exploration with the variance of the REINFORCE estimator. 

Our analysis of the two-timescale update rule builds on
\cite{lin2019gradient}, who analyzed it for minimax problems
$f(x,y)$ where $f$ is nonconvex with respect to $x$
but \emph{concave} with respect to $y$. Compared to this setting, our
nonconvex-nonconcave setup poses additional difficulties. At a
high level, our approach is as follows. First, thanks to the gradient dominance condition for the $x$-player,
to find an $\eps$-suboptimal solution it suffices to ensure that the gradient of
$\Phi(x)\ldef{}\max_{y\in\cY}f(x,y)$ is small. However, since $\Phi$ may not
differentiable, we instead aim to minimize
$\nrm*{\grad\Phi_{\lambda}(x)}_2$, where $\Phi_{\lambda}$ denotes the
Moreau envelope of $\Phi$ (\pref{app:technical-prelim}). If the $y$-player performed a best response at each iteration, a standard
analysis of nonconvex stochastic subgradient descent
\citep{davis_stochastic_2018}, would ensure that
$\nrm*{\grad\Phi_{\lambda}(x\ind{i})}_2$ converges at an $\eps^{-4}$ rate. The crux of our analysis is to argue that, since the $x$ player
operates at a much slower timescale than the $y$-player, the $y$-player approximates a best response in terms of function value. Compared to
\cite{lin2019gradient}, which establishes this property using
convexity for the $y$-player, we use the gradient dominance condition
to bound the $y$-player's immediate suboptimality in terms of the norm of the gradient
of the function $\psi_{t,\lambda}(y)\ldef{}-(-f(x_t,\cdot))_{\lambda}(y)$, then show
that this quantity is small on average using a potential-based argument.

%% file: section_last_iterate.tex
\dfcomment{Legends and axes are way too small in the figure. I am tempted to 1) make the axis ticks bigger, with fewer axes 2) remove legend, since we state it in the text anyway.}\noah{fixed (don't know what you mean by make axis ticks bigger though)}

\dfcomment{Not very clear from the figure's description that a/b and c/d are for two choices of $R$/$S$. Maybe just refer back to the text so this is more clear? It would be good to make the red dot slightly larger also} \noah{fixed}

An important problem left open by our work is to develop independent
policy gradient-type updates that enjoy \emph{last iterate
  convergence}. This property is most cleanly stated in the noiseless
setting, with exact access to gradients: For fixed, constant
learning rates \modified{$\eta_x = \eta_y = \eta$}, we would like that if both learners
independently run the algorithm, their iterates satisfy
\[
\lim_{i\to\infty}x\ind{i}\to\xstar,\quad\text{and}\quad \lim_{i\to\infty}y\ind{i}\to\ystar.
\]
Algorithms with this property have enjoyed intense recent interest for
continuous, zero-sum games \citep{daskalakis2017training,daskalakis2018limit,mertikopoulos2018cycles,daskalakis2019last,liang2019interaction,gidel2019negative,mokhtari2019unified,kong2019accelerated,gidel2019variational,abernethy2019last,AzizianMLG20,GolowichPDO}. These include Korpelevich's extragradient
method \citep{korpelevich1976extragradient}, Optimistic Mirror Descent
(e.g., \cite{daskalakis2017training}), and variants.
For a generic minimax
problem $f(x,y)$, the updates for the extragradient method take the form
\begin{equation}
  \tag{EG}\label{eq:eg}
\begin{aligned}
  &x\ind{i+1}\gets{}\proj_{\cX}\prn*{x\ind{i}-\eta\grad_xf(x\ind{i+1/2},
  y\ind{i+1/2})},\quad\text{and}\quad
  y\ind{i+1}\gets{}\proj_{\cY}\prn*{y\ind{i}+\eta\grad_yf(x\ind{i+1/2},
  y\ind{i+1/2})},\\
  &\text{where}\quad x\ind{i+1/2}\gets{}\proj_{\cX}\prn*{x\ind{i}-\eta\grad_xf(x\ind{i},
  y\ind{i})},\quad\text{and}\quad y\ind{i+1/2}\gets{}\proj_{\cY}\prn*{y\ind{i}+\eta\grad_yf(x\ind{i},
  y\ind{i})}. 
\end{aligned}
\end{equation}
In the remainder of this section we show that while the extragradient
method appears to succeed in simple two-player zero-sum stochastic
games experimentally, establishing last-iterate convergence formally likely requires new tools.  We conclude with an
open problem. 

As a running example, we consider von Neumann's
\emph{ratio} game \citep{neumann1945model}, a very simple stochastic game given by
\begin{equation}
  \label{eq:ratio_game}
  V(x,y) = \frac{\tri*{x,Ry}}{\tri*{x,Sy}},
\end{equation}
where $x\in\Delta(\cA)$, $y\in\Delta(\cB)$,
 $R\in\bbR^{A\times{}B}$, and $S\in\bbR_{+}^{A\times{}B}$, with $\tri*{x,Sy}\geq\zeta$
for all $x\in\Delta(\cA)$, $y\in\Delta(\cB)$. The expression \pref{eq:ratio_game} can be interpreted as the value $V(\pi_x,\pi_y)$ for a stochastic game with a single state, where the immediate reward
for selecting actions $(a,b)$ is $R_{a,b}$, the probability of
stopping in each round is $S_{a,b}$, and both players use the direct parameterization.\footnote{Since there is
  a single state, we drop the dependence on the initial state
  distribution.} Even for this simple game, with exact gradients, we
know of no algorithms with last iterate guarantees.

\begin{figure}[t]
   \subfigure[MVI heatmap]{\includegraphics[scale=0.245]{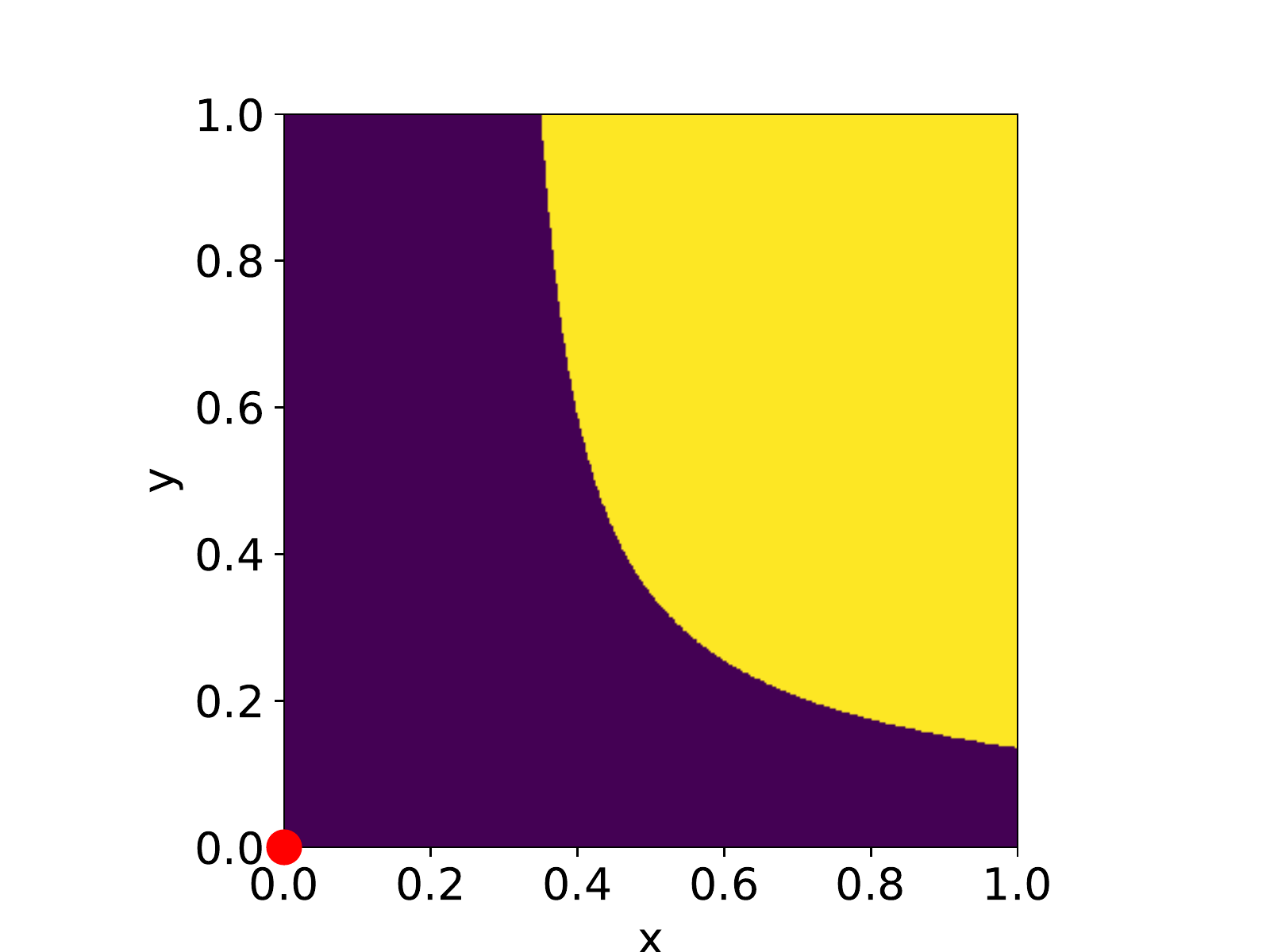}\label{fig:mvi_easy}}
 \subfigure[Convergence of EG]{\includegraphics[scale=0.245]{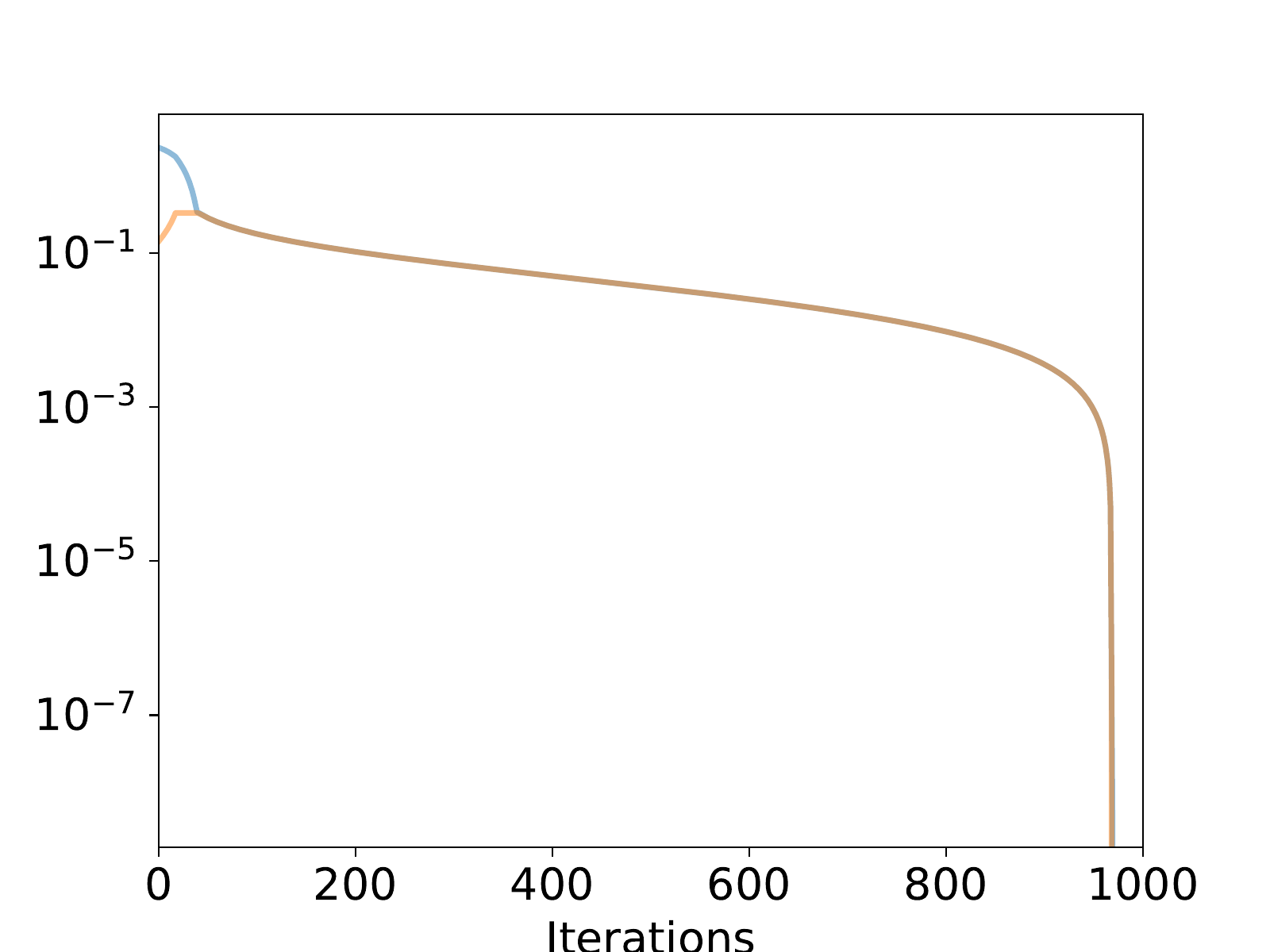}\label{fig:convergence_easy}}
 \centering \subfigure[MVI heatmap]{\includegraphics[scale=0.245]{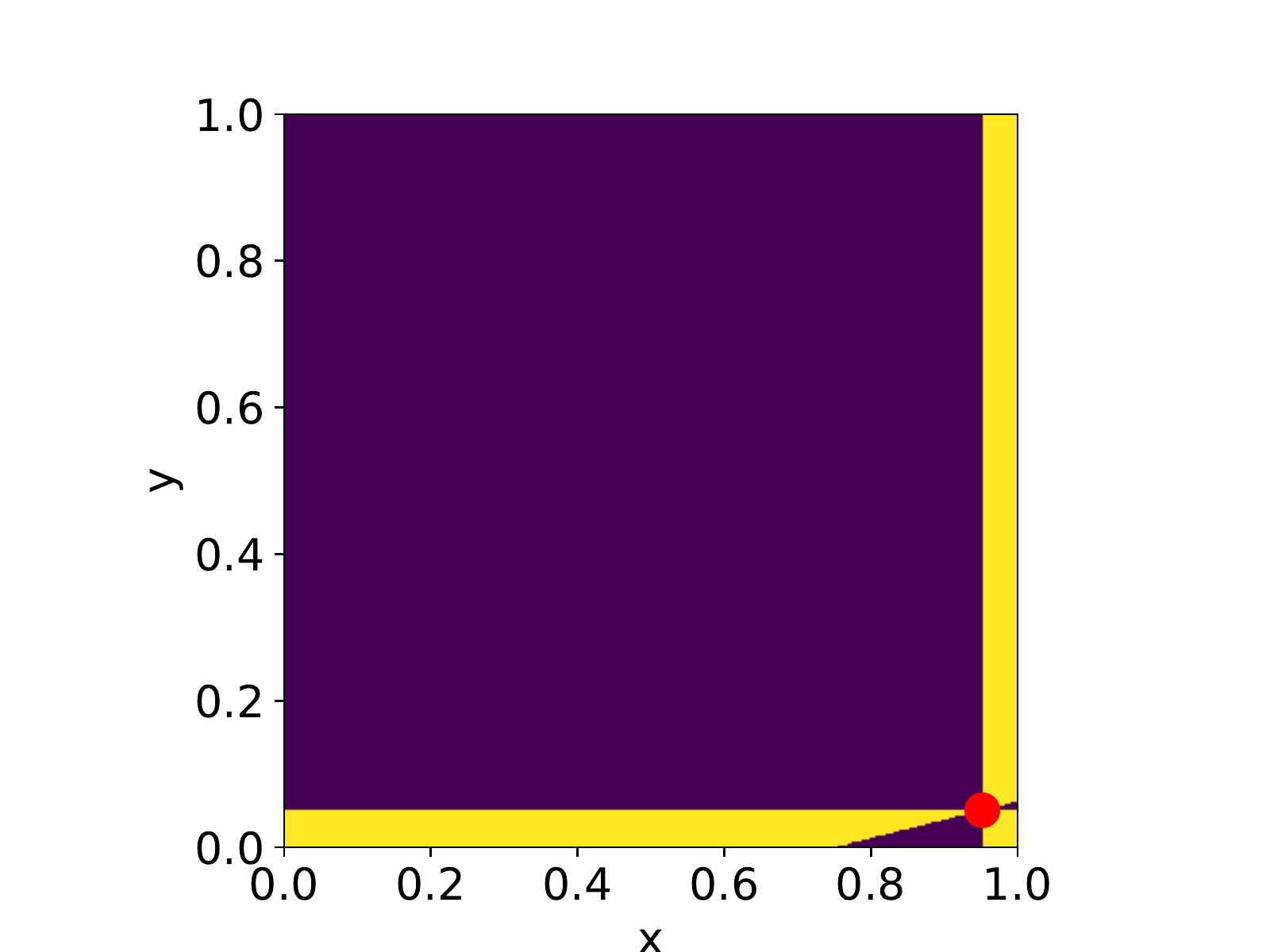}\label{fig:mvi_hard}}
 \subfigure[Convergence of
 EG]{\includegraphics[scale=0.245]{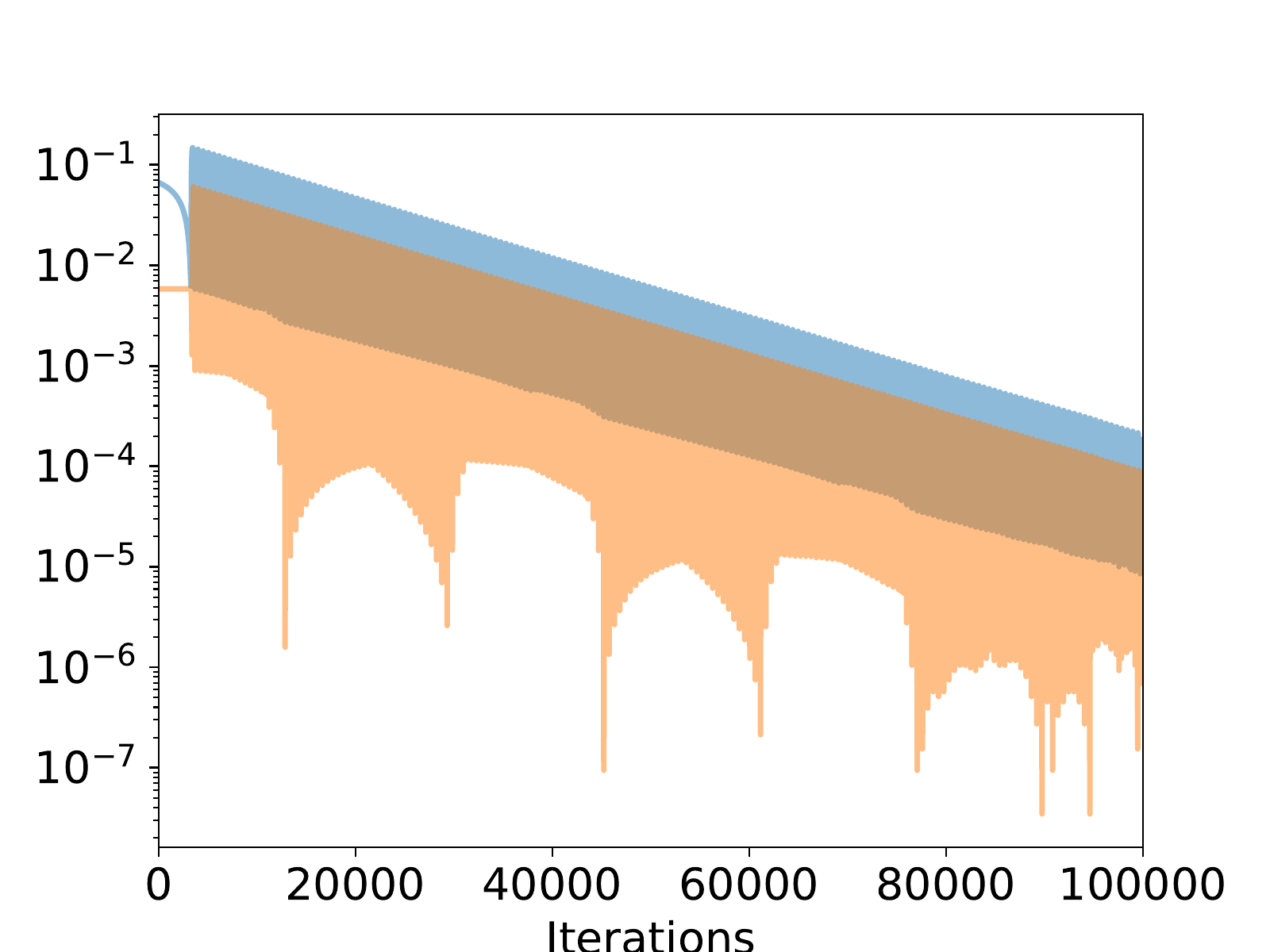} \label{fig:convergence_hard}}
 \caption{Figures (a) and (b) display plots for one ratio game, and Figures (c) and (d) display plots for another; the games' matrices are specified in \pref{app:experimental-details}. Figures (a) and (c) plot the quantity ${\rm sign}(\lng F(z),
   z - z^* \rng)$ for $z \in \Delta^2 \times \Delta^2$, parameterized as $z := (x,1-x,y,1-y)$; yellow denotes
   negative and purple denotes positive. The red dot denotes the
   equilibrium $z^*$. Figures (b) and (d)
   plot convergence of extragradient with learning rate 0.01, %
   initialized at $z_0 := (1,0,1,0)$; note that $z_0$ is inside the region in which the MVI does not hold for each problem. The blue line plots the primal-dual gap $\max_{y'} V(x\^i,y') - \min_{x'} V(x',y\^i)$ and the orange line plots the primal gap $\max_{y'} V(x\^i,y') - V(x^*, y^*)$.} 
\end{figure}

\paragraph{On the MVI condition.}
For nonconvex-nonconcave minimax problems, the only general tool we
are aware of for establishing last-iterate convergence for the 
extragradient method and its relatives is the \emph{Minty Variational
Inequality} (MVI) property
\citep{facchinei2007finite,lin2018solving,mertikopoulos2018stochastic,mertikopoulos2019optimistic,gidel2019variational}. For
$z=(x,y)$ and $F(z) \ldef (\grad_xf(x,y),-\grad_yf(x,y))$, the MVI
property requires that there exists a point $\zstar\in\cZ\ldef\cX\times\cY$
such that
\begin{equation}
  \tag{MVI}\label{eq:mvi}
  \tri*{F(z),z-\zstar}\geq{}0\quad\forall{}z\in\cZ.
\end{equation}
For general minimax problems, the MVI property is typically applied with $\zstar$ as a Nash equilibrium \citep{mertikopoulos2019optimistic}.
We show that this condition fails in stochastic games, even for the
simple ratio game in \pref{eq:AS}
\begin{proposition}
  \label{prop:mvi_counterex}
  Fix $\ep,s\in (0,1)$ with $\eps<\frac{1-s}{2s}$. Suppose we take
\begin{equation}
  \label{eq:AS}
R = \matx{
  -1 & \ep \\
  -\ep & 0
}, \quad\text{and}\quad S = \matx{s & s \\ 1 & 1}.
\end{equation}
Then the ratio game defined by \pref{eq:AS} has the following
properties: (1) there is a unique Nash equilibrium $z^* = (z^*, y^*)$ given by $x^* = y^* =
(0,1)$, (2) $\zeta\geq{}s$, (3) there exists $z = (x,y) \in \Delta(\cA)\times \Delta(\cB)$ so that $
\lng F(z), z - z^* \rng < 0$.\footnote{In fact, for this example the
  MVI property fails for all choices of $\zstar$, not just the Nash equilibrium.}
\end{proposition}
\pref{fig:mvi_easy} plots the sign of $\tri*{F(z),z-\zstar}$ for the game in 
\pref{eq:AS} as a
function of the players' parameters, which changes based on
whether they belong to one of two regions, and
\pref{fig:convergence_easy} shows that extragradient readily converges
to $\zstar$ in spite of the failure of MVI. While this example satisfies the MVI property locally
around $\zstar$, \pref{fig:mvi_hard} shows a randomly
generated game (\pref{app:experimental-details}) for which the MVI property fails to hold even locally. Nonetheless,
\pref{fig:convergence_hard} shows that extragradient converges
for this example, albeit more slowly, and with oscillations. This
leads to our open problem.

\begin{openproblem}
  \label{op:eg}
  Does the extragradient method with constant learning rate have
  last-iterate convergence for the ratio game \pref{eq:ratio_game} for any fixed $\zeta > 0$?
\end{openproblem}
Additional experiments with \emph{multi-state} games generated at random suggest that the extragradient method has last-iterate convergence for general stochastic games with a positive stopping probability. Proving such a convergence result for extragradient or for relatives such as the optimistic gradient method would be of interest not only because it would guarantee last-iterate convergence, but because it would provide an algorithm that is \emph{strongly independent} in the sense that two-timescale updates are not required.

%% file: section_related_work.tex
Issues of independence in MARL have enjoyed extensive investigation. We refer the reader to \citet{zhang2019multi} for a comprehensive overview and discuss some particularly relevant related work below.

\neurips{
\vspace{0.2cm}
  }
\paragraph{Stochastic games.}
Beginning with their introduction by \citet{shapley_stochastic_1953}, there is a long line of work developing computationally efficient algorithms for multi-agent RL in stochastic games \citep{littman_markov_1994,hu2003nash, bu2008comprehensive}. While centralized, coordinated MARL algorithms such as self-play have recently enjoyed some advances in terms of non-asymptotic guarantees \modified{\cite{brafman2002r,wei2017online,bai2020provable,xie2020learning,zhang2020model}}, independent RL has seen less development, with a few exceptions we discuss below.

A recent line of work \citep{srinivasan2018actor,omidshafiei2019neural,lockhart2019computing} shows that for zero-sum extensive form games (EFG), independent policy gradient methods can be formulated in the language of \emph{counterfactual regret minimization} \citep{zinkevich2008regret}, and uses this observation to derive convergence guarantees. Unfortunately, for the general zero-sum stochastic games we consider, reducing to an EFG results in exponential blowup in size with respect to horizon.

\modified{\citet{arslan2017decentralized} introduce an algorithm for learning stochastic games which can be viewed as a 2-timescale method and show convergence (though without rates) in a somewhat different setting from ours.} \citet{perolat2018actor} provide asymptotic guarantees for an independent two-timescale actor-critic method in zero-sum stochastic games with a ``simultaneous-move multistage''  structure in which each state can only be visited once. Our result is somewhat more general since it works for arbitrary infinite-horizon stochastic games, and is non-asymptotic.

\citet{zhang2019policy,bu2019global} recently gave global convergence results for policy gradient methods in two-player zero-sum linear-quadratic games. These results show that if the \minplayer follows policy gradient updates and the \maxplayer follows the best response at each timestep, the \minplayer will converge to a Nash equilibrium. These results do not satisfy the independence property defined in \pref{sec:independent}, since they follow an inner-loop/outer-loop structure and assume exact access to gradients of the value function. Interestingly, \citet{mazumdar2019policy} show that for general-sum linear-quadratic games, independent policy gradient methods can fail to converge even locally.

\modified{Two concurrent works also develop provable independent learning algorithms for stochastic games. \citet{lee2020linear} show that the optimistic gradient algorithm obtains linear rates in the full-information and finite-horizon (undiscounted) setting, where the transition probability function $P$ is known and we have exact access to gradients. Their rate depends on the constant in a certain restricted secant inequality; this constant can be arbitrarily small even in the setting of matrix games (i.e., a single state, $\zeta = 1$, and fixed $A,B$), which causes the rate to be arbitrarily slow. In a setting very similar to that of this paper, \citet{bai2020near} propose a model-free upper confidence bound-based algorithm, Nash V-learning, which satisfies the independent learning requirement and has near-optimal sample complexity, achieving superior dependence to \pref{thm:main} on the parameters $S,A,B,\zeta$, as well as no dependence on $C_{\MG}$. However, their work has the limitation of only learning non-Markovian policies, whereas the policies learned by 2-timescale SGDA are Markovian (i.e., only depend on the current state).} %

\neurips{
  \vspace{0.2cm}
}
\paragraph{Minimax optimization and (non-monotone) variational inequalities.}
Since the objective $\Vrho(x,y)$ is continuous, a natural approach to minimizing it is to appeal to black-box algorithms for \nonconvex-\nonconcave minimization, and more broadly non-montone variational inequalities. In particular, the gradient dominance condition implies that all first-order stationary points are Nash equilibria. Unfortunately, compared to the single-player setting, where many algorithms such as gradient descent find first-order stationary points for arbitrary smooth, \nonconvex functions, existing algorithms for non-monotone variational inequalities all require additional assumptions that are not satisfied in our setting. \citet{mertikopoulos2019optimistic} give convergence guarantees for non-monotone variational inequalities satisfying the so-called MVI property, which we show fails even for single-state zero-sum stochastic games (\pref{sec:last_iterate}). \citet{yang2020global} give an alternating gradient descent algorithm which succeeds for \nonconvex-\nonconcave games under a two-sided \pllong condition, but this condition (which leads to linear convergence) is also not satisfied in our setting. Another complementary line of work develops algorithms for nonconvex-concave problems \citep{rafique2018non,thekumparampil2019efficient,lu2019hybrid,nouiehed2019solving,kong2019accelerated,lin2019gradient}.
\neurips{
  \vspace{0.65cm}
  }

%% file: section_discussion.tex
We presented the first independent policy gradient
algorithms for competitive reinforcement learning in zero-sum
stochastic games. We hope our results will serve as a starting
point for developing a more complete theory for independent
reinforcement learning in competitive RL and multi-agent reinforcement
learning. Beyond \pref{op:eg}, there are a number of questions raised by
our work.

  \modified{\citet{efroni2020optimistic} have recently shown how to improve the convergence rates for policy gradient algorithms in the single-agent setting by incorporating optimism. Finding a way to use similar techniques in the multi-agent setting under the independent learning requirement could be another promising direction for future work.}

  Many games of interest are not zero-sum, and may involve more
  than two players or be cooperative in nature. It would be useful to
  extend our results to these settings, albeit likely for weaker
  solution concepts, and to derive a tighter understanding of the
  optimization geometry for these settings.

  On the technical side, there are a number of immediate technical
  extensions of our results which may be useful to pursue, including (1) extending to linear function approximation, (2) extending to
  other policy parameterizations such as soft-max, and (3)
  actor-critic and natural policy gradient-based variants
  \citep{agarwal2019optimality}.

\dfcomment{Make sure references below are not out of date (eg, arxiv versions where the conference version has already appeared)}\noah{done}

%% file: appendix_scratch.tex
\paragraph{Q-functions and advantage functions.}
For policies $\pi_1,\pi_2$, we let $Q^{\pi_1,\pi_2}(s,a,b)$ denote the {\it $Q$-value function}:
$$
Q^{\pi_1,\pi_2}(s,a,b) := \E_{\pi_1,\pi_2} \left[ \sum_{t=0}^T R(s_t, a_t, b_t) | s_0 = s, a_0 = a, b_0 = b \right],
$$
and we let $A^{\pi_1,\pi_2}(s,a,b) = Q^{\pi_1,\pi_2}(s,a,b) - V_s(\pi_1,\pi_2)$ denote the \emph{advantage function}.

Throughout this section we abbreviate $\Vrho(x,y)=\Vrho(\pi_x,\pi_y)$, where $\pi_x$ and $\pi_y$ use the $\veps$-greedy direct parameterization in \pref{ass:greedy}.

%% file: appendix_algorithms.tex
\subsection{Full Version of Theorem \ref*{thm:main} and Proof}
\newcommand{\Deltaphi}{\Delta_{\Phi}}
The full version of \pref{thm:main} is as follows.
\begin{thmmod}{thm:main}{a}
  \label{thm:main_full}
Let $\eps>0$ be given. Suppose both players follow
  the independent policy gradient scheme \pref{eq:pg} with the parametrization in \pref{ass:greedy}. If the
  learning rates satisfy $\eta_x= \Theta \left( \frac{\ep_0^{10.5} \zeta^{44.5}}{C_\MG^{15.5}(A \vee B)^{9.75} S^{0.75}}\right)$ and
  $\eta_y = \Theta \left( \frac{\ep^6\zeta^{27}}{C_\MG^9(A \vee B)^6 \sqrt{S}}\right)$ and $\gdx = \Theta \left( \frac{\zeta^3 \cdot \ep}{\sqrt{S} \sqrt{A \vee B} C_\MG} \right)$ and $\gdy = \Theta \left( \frac{\zeta^8 \ep^2}{C_\MG^3 (A \vee B) \sqrt{S}}\right)$, then we are guaranteed that
  \begin{align}
    \label{eq:main}
    \textstyle    \En\brk*{\frac{1}{N}\sum_{i=1}^{N}\max_{\pi_2}V_{\rho}(\pi_{x\ind{i}},\pi_2)} - \min_{\pi_1}\max_{\pi_2}V_{\rho}(\pi_1,\pi_2) \leq{} \eps
  \end{align}
  after $N = O \left( \frac{(A \vee B)^{10.75} S^{1.25} C_\MG^{17.5}}{\ep^{12.5}\zeta^{48.5}}\right)$ episodes.
\end{thmmod}
\dfcomment{it would be nice to make the eqs in the proof below fit in the margins}\noah{fixed}
\begin{proof}[\pfref{thm:main_full}]
  This result is essentially an immediate consequence of
  \pref{thm:sgda}. As a first step, we observe that 
  $D_{\cX},D_{\cY}\leq{}\sqrt{S}$. %
Next, we observe that from \pref{lem:mg-smoothness}, both players have
$\ls\ldef{}\frac{4 A\vee{}B}{\zeta^3}$-jointly-Lipschitz gradients:
  \begin{align*}
\left\| \grad_x V_{\rho}(x,y) - \grad_x V_{\rho}(x', y') \right\|_2 \leq \frac{4(1-\zeta) A}{\zeta^3} \| (x,y) - (x',y') \|_2, \\
\left\| \grad_y V_{\rho}(x,y) - \grad_y V_{\rho}(x', y') \right\|_2 \leq \frac{4(1-\zeta) B}{\zeta^3} \| (x,y) - (x',y') \|_2. 
  \end{align*}
  Similarly, by \pref{prop:gradients}, the function value $V_{\rho}(x,y)$ is $L\ldef{}\frac{2\sqrt{A\vee{}B}}{\zeta^2}$-Lipschitz:
  \begin{align*}
    \nrm*{\grad{}_x\Vrho(x,y)}\leq{}\frac{\sqrt{A}}{\zeta^{2}},\quad\text{and}\quad \nrm*{\grad{}_y\Vrho(x,y)}\leq{}\frac{\sqrt{B}}{\zeta^{2}}.
  \end{align*}
  Note that $L/\ell \leq 1$ since $A\wedge B\wedge 1/\zeta \geq 1$. 
  \pref{lem:reinforce_variance} guarantees that both players have
  bounded variance:
\begin{equation}  \En_{\pi_x,\pi_y}\nrm*{\wh{\grad}_x-\grad{}_x\Vrho(x,y)}^{2}\leq{}24\frac{A^{2}}{\veps_x\zeta^{4}},\quad\text{and}\quad
    \En_{\pi_x,\pi_y}\nrm*{\wh{\grad}_y-\grad{}_y\Vrho(x,y)}^{2}\leq{}24\frac{B^{2}}{\veps_y\zeta^{4}}.
  \end{equation}
  \pref{lem:gd_greedy_full} guarantees that the gradient domination conditions \ref{it:y-pl} and \ref{it:x-pl} of  \pref{asm:lip-smooth-f} are satisfied with $\mu_\sy = \mu_\sx = \zeta/C_\MG$, and additive components $\gdx, \gdy$ equal to $\frac{2\gdx}{\zeta^2}$ and $\frac{2\gdy}{\zeta^2}$, respectively:
    \begin{align*}
     \frac{\zeta}{C_\MG} \left( \Vrho(x,y) - \min_{x'} \Vrho(x',y)\right) - \frac{2\gdx}{\zeta^2}
    &  \leq{}
\max_{\bar{x}\in\Delta(\cA)^{\abs*{\cS}}}\tri*{\grad_{x}\Vrho(x,y),x-\bar{x}},
  \end{align*}
  and
  \begin{align*}
   \frac{\zeta}{C_\MG} \left( \max_{y'} \Vrho(x,y') - \Vrho(x,y)\right) - \frac{2\gdy}{\zeta^2}
  &  \leq{}  \max_{\bar{y}\in\Delta(\cB)^{\abs*{\cS}}}\tri*{\grad_{y}\Vrho(x,y),\bar{y}-y}.
  \end{align*}
  Note that $L/\ell < 1$. 
  By \pref{thm:sgda}, for a desired accuracy level $\ep > 0$, if we set:
 {\begin{align}
      \label{eq:prefinal-etay}
  \eta_\sy &= \Theta \left( \frac{\ep^4 \zeta^{15}\gdy}{C_\MG^2(A \vee B)^5} \right) =\Theta \left( \frac{\ep^4 (\zeta/C_\MG)^2}{((A\vee B)/\zeta^3)^3 ((A\vee B)/\zeta^4 + B^2 / (\gdy\zeta^4))}\right),\\
  \label{eq:prefinal-etax}
           \eta_\sx &= \Theta \left(\frac{\ep^8 \zeta^{27}\gdy \sqrt{\gdx}}{C_\MG^4(A \vee B)^{8.5}} \right) \\
           &= \Theta \left( \frac{\ep^8 (\zeta/C_\MG)^4}{((A\vee B)/\zeta^3)^5 \frac{\sqrt{A\vee B}}{\zeta^2} \cdot ((A\vee B)/\zeta^4 + B^2 / (\gdy\zeta^4)) \sqrt{(A\vee B)/\zeta^4 + A^2 / (\gdx \zeta^4)}} \right),\nonumber\\
        \label{eq:prefinal-N}
N              & \geq \Omega \left( \frac{\sqrt{S} \cdot \frac{A \vee B}{\zeta^2}}{\ep^2 \eta_\sx}\right),
\end{align}}
we have that
\begin{align}
  \neurips{&~~~~~} \frac{1}{N+1} \sum_{i=0}^N\max_{y} V_\rho(x\^i, y) - \min_{x} \max_y V_\rho(x,y) \notag\neurips{\\}
  & \leq O \left( \frac{\ep C_\MG}{\zeta} + \frac{\gdx C_\MG}{\zeta^3} + \frac{\sqrt{(A \vee B)/\zeta^3} \sqrt{\gdy/\zeta^2} (C_\MG/\zeta)}{\sqrt{\zeta/C_\MG}}
                                                                         \right)\notag
  \\
  \label{eq:gap-ub}
  & \leq O \left( \frac{\ep C_\MG}{\zeta} + \frac{\gdx C_\MG}{\zeta^3} + \frac{\sqrt{\gdy}C_\MG^{1.5}\sqrt{A\vee B}}{\zeta^4} \right).
\end{align}
Recall that $x \mapsto \pi_x, y \mapsto \pi_y$ denote the $\gdx$- and $\gdy$-greedy parametrizations, respectively (where $\gdx, \gdy$ are as in the statement of \pref{thm:main_full}). Then for any $\pi_1 : \MS \ra \Delta(\MA)$ (respectively, $\pi_2 : \MS \ra \Delta(\MB)$), there is some $x \in \Delta(\MA)^S$  (respectively, $y \in \Delta(\MB)^s$) so that $\|\pi_1 - \pi_x\| \leq 2\sqrt{S}\gdx$ (respectively, $\|\pi_2 - \pi_y\| \leq 2\sqrt{S}\gdy$) for each $s \in \MS, a \in \MA$ (respectively, $b \in \MB$). Moreover, recall that \pref{prop:gradients} shows that the function $(x,y) \mapsto \Vrho(x,y)$ is $L$-Lipschitz for {\it any} $\varepsilon$-greedy parametrization, in particular for the one given by $\gdx = \gdy = 0$. 
Thus, 
\begin{align*}
  & \left| \left(   \frac{1}{N+1} \sum_{i=0}^N\max_{y} V_\rho(x\^i, y) - \min_{x} \max_y V_\rho(x,y) \right) - \left(   \frac{1}{N+1} \sum_{i=0}^N\max_{\pi_2} V_\rho(\pi_{x\^i}, \pi_2) - \min_{\pi_1} \max_{\pi_2} V_\rho(\pi_1, \pi_2) \right)\right|\\
  & \leq O\left(\frac{\sqrt{S} (\gdx \vee \gdy)\sqrt{A \vee B}}{\zeta^2}\right).
\end{align*}
From \pref{eq:gap-ub} it now follows that
{\small$$
 \frac{1}{N+1} \sum_{i=0}^N\max_{\pi_2} V_\rho(\pi_{x\^i}, \pi_2) - \min_{\pi_1} \max_{\pi_2} V_\rho(\pi_1, \pi_2) \leq O \left( \frac{\ep C_\MG}{\zeta} + \frac{\gdx C_\MG}{\zeta^3} + \frac{\sqrt{\gdy}C_\MG^{1.5}\sqrt{A\vee B}}{\zeta^4} + \frac{\sqrt{S} (\gdx \vee \gdy)\sqrt{A \vee B}}{\zeta^2}\right).
 $$}
 To achieve a desired accuracy level $\ep_0$, it follows from \pref{eq:prefinal-etay}, \pref{eq:prefinal-etax}, and \pref{eq:prefinal-N} that if we set
 $$
\gdx = \Theta \left( \frac{\zeta^3 \cdot \ep_0}{\sqrt{S} \sqrt{A \vee B} C_\MG} \right), \qquad \gdy = \Theta \left( \frac{\zeta^8 \ep_0^2}{C_\MG^3 (A \vee B) \sqrt{S}}\right), \qquad \ep \leq O \left( \frac{\zeta \ep_0}{C_\MG} \right),
$$
and
\begin{align}
  \eta_\sy & = \Theta\left( \frac{\ep_0^6\zeta^{27}}{C_\MG^9(A \vee B)^6 \sqrt{S}}\right) =\Theta \left( \frac{\left(\zeta \ep_0 / C_\MG\right)^4 \zeta^{15} \cdot  \frac{\zeta^8 \ep_0^2}{C_\MG^3 (A \vee B) \sqrt{S}}}{C_\MG^2(A \vee B)^5} \right),\nonumber\\
  \eta_\sx &= \Theta \left( \frac{\ep_0^{10.5} \zeta^{44.5}}{C_\MG^{15.5}(A \vee B)^{9.75} S^{0.75}}\right) =  \Theta \left(\frac{(\zeta \ep_0 / C_\MG)^8 \zeta^{27}\left(\frac{\zeta^8 \ep_0^2}{C_\MG^3 (A \vee B) \sqrt{S}}\right) \sqrt{\frac{\zeta^3 \cdot \ep_0}{\sqrt{S} \sqrt{A \vee B} C_\MG}}}{C_\MG^4(A \vee B)^{8.5}} \right),\nonumber\\
  N &= \Omega \left( \frac{(A \vee B)^{10.75} S^{1.25} C_\MG^{17.5}}{\ep_0^{12.5}\zeta^{48.5}}\right) \geq \Omega \left( \frac{\sqrt{S} \cdot \frac{A \vee B}{\zeta^2}}{(\zeta \ep_0 / C_\MG)^2 \eta_\sx}\right)\nonumber,
\end{align}
then we have
$$
 \frac{1}{N+1} \sum_{i=0}^N\max_{\pi_2} V_\rho(\pi_{x\^i}, \pi_2) - \min_{\pi_1} \max_{\pi_2} V_\rho(\pi_1, \pi_2) \leq \ep_0.
$$
\end{proof}

\subsection{Proofs for Additional Results}
\begin{proof}[\pfref{prop:concentratability}]
  We define the following game $\MG$ with state space $\MS :=
  \{1,2,3,4,5\}$, action spaces $\MA = \MB = \{0,1\}$, and any stopping probability $\zeta > 0$.

  The transitions and rewards are as follows:
  \begin{itemize}
  \item In state 1, with probability $\zeta$, the game stops. Conditioned on not stopping:
    \begin{itemize}
    \item If actions $(0,0)$ are taken, the game moves to state 2.
    \item If actions $(0,1)$ are taken, the game moves to state 3.
    \item If actions $(1,0)$ are taken, the game moves to state 4.
    \item If actions $(1,1)$ are taken, the game moves to state 5.
    \end{itemize}

    Both players receive 0 reward in state 1.
  \item In state $2 \leq i \leq 5$, the game stops with probability
    $\zeta$, and otherwise always moves back to state 1. Furthermore, player 1 receives reward $i-1$ and player 2 receives reward $1-i$ (regardles of their actions).
  \end{itemize}
  Let the initial state distribution $\rho = \delta_{1}$ be defined by $$\rho(1) = \rho(2) = \rho(4) = \rho(5) = 1/4,\quad\text{and}\quad \rho(3) = 0.$$

  Clearly, the value $V_\rho({\pi_1, \pi_2})$ of the game depends only on the policies at state 1, i.e., $\pi_1(\cdot | 1), \pi_2(\cdot | 1)$. If $\pi_1(0|1) = \pi_2(1|1) = 1$, then certainly $d^{\pi_1, \pi_2}_\rho(3) > 0$, and therefore $\max_{\pi_1, \pi_2}\nrm*{\frac{d_\rho^{\pi_1,\pi_2}}{\rho}}_\infty$ is infinite.

On the other hand, let us now consider the best-response policies:
\begin{itemize}
\item For any policy $\pi_1$ of the \minplayer, all policies $\pi_2 \in \Pi_2^*(\pi_1)$ of the \maxplayer satisfy $\pi_2(0|1) = 1$. This follows since player 2 prefers state 2 to state 3, and state 4 to state 5.

  In particular, for any pair $(\pi_1, \pi_2)$ with $\pi_2 \in \Pi_2^*(\pi_1)$, we have that $d_\rho^{\pi_1, \pi_2}(3) =  0$. %
\item For any policy $\pi_2$ of the \maxplayer, all policies $\pi_1 \in \Pi_1^*(\pi_1)$ of the \minplayer satisfy $\pi_1(1|0) = 1$. This follows since player 1 prefers state 4 to state 2, and state 5 to state 3.

  In particular, for any pair $(\pi_1, \pi_2)$ with $\pi_1 \in \Pi_1^*(\pi_2)$, we have that $d_\rho^{\pi_1, \pi_2}(3) = 0$. %
\end{itemize}
It follows that
$$
C_\MG = \max\crl*{\max_{\pi_2}\min_{\pi_1\in\Pistar_1(\pi_2)}\nrm*{\frac{d^{\pi_1,\pi_2}_{\rho}}{\rho}}_{\infty}, \max_{\pi_1}\min_{\pi_2\in\Pistar_2(\pi_1)}\nrm*{\frac{d^{\pi_1,\pi_2}_{\rho}}{\rho}}_{\infty}} < \infty,
$$
which completes the proof of the proposition.
\end{proof}

\subsection{Supporting Lemmas}

\begin{lemma}
  \label{lem:reinforce_variance}
  Suppose that players follow that $\veps$-greedy direct
  parameterization in \pref{ass:greedy} with parameters $\veps_x$ and
  $\veps_y$. Given parameters $x\in\Delta(\cA)^{\abs*{\cS}},
  y\in\Delta(\cA)^{\abs*{\cS}}$ suppose the players estimate their
  gradients using the REINFORCE estimator:
  \begin{equation}
    \wh{\grad}_x \ldef{} R_T\sum_{t=0}^{T}\grad\log\pi_x(a_t\mid{}s_t),\quad
    \wh{\grad}_y \ldef{} R_T\sum_{t=0}^{T}\grad\log\pi_y(b_t\mid{}s_t),
  \end{equation}
  under trajectories obtained by following $\pi_x$ and $\pi_y$.
  Then we have
  \begin{equation}
    \label{eq:reinforce_unbiased}
\E_{\pi_x, \pi_y} \wh{\grad}_x  = \grad_x V_\rho(\pi_x, \pi_y), \qquad \E_{\pi_x, \pi_y} \wh{\grad}_y = \grad_y V_\rho(\pi_x, \pi_y),
\end{equation}
and
  \begin{equation}
    \label{eq:reinforce_variance}
\En_{\pi_x,\pi_y}\nrm*{\wh{\grad}_x-\grad{}_x\Vrho(x,y)}^{2}\leq{}24\frac{A^{2}}{\veps_x\zeta^{4}},\quad\text{and}\quad
    \En_{\pi_x,\pi_y}\nrm*{\wh{\grad}_y-\grad{}_y\Vrho(x,y)}^{2}\leq{}24\frac{B^{2}}{\veps_y\zeta^{4}}.
  \end{equation}
\end{lemma}
\begin{proof}[\pfref{lem:reinforce_variance}]
  We carry the calculation out for the $x$ player, as the $y$ player
  follows an identical argument. We start by proving that the gradient estimator is unbiased (i.e., \pref{eq:reinforce_unbiased}). Let $\MT$ denote the (infinite) set of all possible trajectories, and for a trajectory $\tau = (s_t, a_t, b_t, r_t)_{0 \leq t \leq T}$, in $\MT$, let $R(\tau) := \sum_{t=0}^T r_t$ denote the total reward associated with $\tau$, and for policies $\pi_1, \pi_2$, let
  $$
  \pr^{\pi_1, \pi_2}(\tau) := \prod_{t=0}^T \pi_1(a_t | s_t) \pi_2(b_t | s_t) P(s_{t+1} | s_t, a_t,b_t)
  $$
  be the probability of realizing $\tau$. (Here, we let $s_{T+1}$ denote the event that the game stops at time $T$) Let $T(\tau)$ denote the last time step of trajectory $\tau$. Then
  \begin{align*}
    & \grad_x V_\rho(x,y) \\
    &= \grad_x \sum_{\tau \in \MT} R(\tau) \pr^{\pi_x, \pi_y}(\tau) \\
    &= \sum_{\tau \in \MT} R(\tau) \grad_x \pr^{\pi_x, \pi_y}(\tau) \\
    &= \sum_{\tau \in \MT} R(\tau) \pr^{\pi_x, \pi_y}(\tau) \grad_x \log \pr^{\pi_x, \pi_y}(\tau) \\
    &= \sum_{\tau \in \MT} R(\tau) \pr^{\pi_x, \pi_y}(\tau) \grad_x \left( \sum_{t=0}^{T(\tau)} \log \pi_x(a_t | s_t) + \log \pi_y(b_t | s_t) \right) \\
    &= \E_{\pi_x, \pi_y} \left[ \left(\sum_{t=0}^T r_t\right) \sum_{t=0}^T \grad_x \log \pi_x(a_t | s_t)\right]\\
    &= \E_{\pi_x, \pi_y} \left[ \wh{\grad}_x \right].
  \end{align*}
  A similar calculation shows that $\E_{\pi_x, \pi_y} \left[ \wh{\grad}_y \right] = \grad_y V_\rho(x,y)$.

  We proceed to bound the variance of the gradient estimator (i.e.,
  establish \pref{eq:reinforce_variance}). Since the gradient estimator is
  unbiased, we have
  \begin{align*}
        \En_{\pi_x,\pi_y}\nrm*{\wh{\grad}_x-\grad{}_x\Vrho(x,y)}^{2}
    \leq{}     \En_{\pi_x,\pi_y}\nrm*{\wh{\grad}_x}^{2}\leq{}     \En_{\pi_x,\pi_y}\nrm*{R_T\sum_{t=0}^{T}\grad\log\pi_x(a_t\mid{}s_t)}^{2}.
  \end{align*}
Next,
we have
  \begin{align*}
    \En_{\pi_x,\pi_y}\nrm*{R_T\sum_{t=0}^{T}\grad\log\pi_x(a_t\mid{}s_t)}^{2}
    &\leq{}\En_{\pi_x,\pi_y}\brk*{(T+1)^2\nrm*{\sum_{t=0}^{T}\grad\log\pi_x(a_t\mid{}s_t)}^{2}}\\
    &\leq{}\En_{\pi_x,\pi_y}\brk*{(T+1)^3\sum_{t=0}^{T}\nrm*{\grad\log\pi_x(a_t\mid{}s_t)}^{2}}\\
    &=\En_{\pi_x,\pi_y}\brk*{(T+1)^3\sum_{t=0}^{T}\sum_{s,a}(1-\veps_x)^2\indic\crl*{s=s_t,a=a_t}\frac{1}{\pi_x^{2}(a\mid{}s)}},\\
&\leq{}\En_{\pi_x,\pi_y}\brk*{(T+1)^3\sum_{t=0}^{T}\sum_{s,a}\indic\crl*{s=s_t,a=a_t}\frac{1}{\pi_x^{2}(a\mid{}s)}},
  \end{align*}
  where the equality is a consequence of the direct parameterization. We further simplify as 
\begin{align*}
  \En_{\pi_x,\pi_y}\brk*{(T+1)^3\sum_{t=0}^{T}\sum_{s,a}\indic\crl*{s=s_t,a=a_t}\frac{1}{\pi_x^{2}(a\mid{}s)}}
  &=
    \sum_{t=0}^{\infty}\sum_{s,a}\En_{\pi_x,\pi_y}\brk*{\indic_{t\leq{}T}\indic\crl*{s=s_t,a=a_t}\frac{1}{\pi_x^{2}(a\mid{}s)}(T+1)^3}\\
  &=
\sum_{t=0}^{\infty}\sum_{s,a}\En_{\pi_x,\pi_y}\brk*{\indic_{t\leq{}T}\indic\crl*{s=s_t,a=a_t}\frac{1}{\pi_x^{2}(a\mid{}s)}(T+1)^3}\\
    &=
      \sum_{t=0}^{\infty}\sum_{s,a}\En_{\pi_x,\pi_y}\brk*{\indic_{t\leq{}T}\indic\crl{s_t=s}\frac{1}{\pi_x(a\mid{}s)}(T+1)^3}\\
  &\leq{}
    \frac{A}{\veps_x}\sum_{t=0}^{\infty}\sum_{s,a}\En_{\pi_x,\pi_y}\brk*{\indic_{t\leq{}T}\indic\crl{s_t=s}(T+1)^3}\\
  &\leq
\frac{A^2}{\veps_x}\En_{\pi_x,\pi_y}\brk*{(T+1)^4}.
\end{align*}
To conclude, we observe that
\[
\En_{\pi_x,\pi_y}\brk*{(T+1)^4} \leq{}
\sum_{t=0}^{\infty}(1-\zeta)^{t}\zeta{}(t+1)^{4}\\
=\frac{\zeta}{1-\zeta}\sum_{t=1}^{\infty}(1-\zeta)^{t}t^{4}\\
\leq{} \frac{24}{\zeta^{4}}.
\]
\end{proof}

Define, for any $s_0 \in \MS$ and policies $\pi_1, \pi_2$,
  $$
\tilde d_{s_0}^{\pi_1, \pi_2}(s) := \sum_{t \geq 0} \pr^{\pi_1, \pi_2} (s_t = s | s_0),
$$
and $\tilde{d}_\rho^{\pi_1, \pi_2}(s) := \E_{s_0 \sim \rho}[\tilde d_{s_0}^{\pi_1, \pi_2}(s)]$ be the {\it un-normalized} state visitation distribution. Also let $\pr^{\pi_1, \pi_2}(\tau | s_0)$ (respectively, $\pr^{\pi_1, \pi_2}(\tau | \rho)$) be the distribution of trajectories $\tau$ given policies $\pi_1, \pi_1$ and initial state $s_0$ (respectively, initial state distribution $\rho$).
\begin{proposition}
  \label{prop:gradients}
In the direct parameterization with $\veps$-greedy
exploration (\pref{ass:greedy}), we have, for all $s \in \MS, a \in \MA, b\in \MB$,
\begin{align*}
  \frac{\partial V_\rho(x,y)}{\partial x_{s,a}} &= (1 - \gdx) \tilde d_\rho^{\pi_x, \pi_y}(s) \E_{b \sim \pi_y(\cdot | s)} \left[ Q^{\pi_x, \pi_y}(s,a,b) \right] \\
  \frac{\partial V_\rho(x,y)}{\partial y_{s,b}} &= (1 - \gdy) \tilde d_\rho^{\pi_x, \pi_y}(s) \E_{a \sim \pi_x(\cdot | s)} [ Q^{\pi_x, \pi_y}(s,a,b)].
\end{align*}
and so it follows that for all $\gdx, \gdy \geq 0$,
\begin{align*}
\left|\frac{\partial V_\rho}{\partial x_{s,a}}(x,y)\right| &\leq  \frac{1}{\zeta} d_\rho^{\pi_x,\pi_y}(s)\left| \E_{b \sim y(\cdot | s)}[Q^{\pi_x,\pi_y}(s,a,b)]\right| \\
\left|\frac{\partial V_\rho}{\partial y_{s,b}}(x,y) \right|&\leq  \frac{1}{\zeta} d_\rho^{\pi_x,\pi_y}(s)\left| \E_{a \sim x(\cdot | s)} [Q^{\pi_x,\pi_y}(s,a,b)]\right|.
\end{align*}
As a consequence, $\nrm*{\grad{}_x\Vrho(x,y)}\leq{}\frac{\sqrt{A}}{\zeta^{2}}$ and $\nrm*{\grad{}_y\Vrho(x,y)}\leq{}\frac{\sqrt{B}}{\zeta^{2}}$.
\end{proposition}
\begin{proof}[Proof of \pref{prop:gradients}]
Note that for any $s \in \MS$, $\tilde d_\rho^{\pi_1, \pi_2}(s) \leq \frac{d_\rho^{\pi_1,\pi_2}(s)}{\zeta}$. 

  Fix any initial state $s_0 \in \MS$. Note that
  \begin{align*}
    & \grad_x V_{s_0}(x, y) \\
    & = \grad_x \sum_{a_0 \in \MA} \pi_x(a_0 | s_0) \E_{b_0 \sim \pi_y(\cdot | s_0)} [Q^{\pi_x, \pi_y}(s_0, a_0, b_0)] \\
    & = \sum_{a_0} (\grad_x \pi_x(a_0 | s_0)) \E_{b_0 \sim \pi_y(\cdot | s_0)} [Q^{\pi_x, \pi_y}(s_0, a_0, b_0)] + \sum_{a_0} \pi_x(a_0 | s_0) \E_{b_0 \sim \pi_y(\cdot | s_0)}[\grad_x Q^{\pi_x, \pi_y}(s_0, a_0, b_0)] \\
    & = \sum_{a_0} \pi_x(a_0 | s_0) (\grad_x \log \pi_x(a_0 | s_0)) \E_{b_0 \sim \pi_y(\cdot | s_0)} [Q^{\pi_x, \pi_y}(s_0, a_0, b_0)] \\
    &~~~~+\sum_{a_0} \pi_x(a_0 | s_0) \E_{b_0 \sim \pi_y(\cdot | s_0)} \left[ \sum_{s_1} P(s_1 | s_0, a_0, b_0) \grad_x V_{s_1}(x,y) \right]\\
    & =\E_{\tau \sim \pr^{\pi_x, \pi_y}(\cdot | s_0)} %
      \left[ (\grad_x \log \pi_x(a_0 | s_0)) Q^{\pi_x, \pi_y}(s_0, a_0, b_0) \right] \\
    &~~~~+\E_{\tau \sim \pr^{\pi_x, \pi_y}(\cdot | s_0)} \left[ \indic_{T \geq 1} \grad_x V_{s_1}(x,y) \right].
  \end{align*}
Note that the above calculation holds also when $s_0$ is replaced with any distribution $\rho \in \Delta(\MS)$. It follows by induction and the fact that $\pr[T \geq t] \leq (1-\zeta)^t$ for any $t \geq 0$ that for any $\rho \in \Delta(\MS)$,
  \begin{align*}
    \grad_x V_\rho(x,y) &= \E_{\tau \sim \pr^{\pi_x, \pi_y}(\cdot | s_0)} \left[ \sum_{t=0}^T (\grad_x \log \pi_x(a_t | s_t)) Q^{\pi_x, \pi_y}(s_t, a_t, b_t)\right] \\
                        &= \sum_{s \in \MS} \E_{a\sim \pi_x(\cdot | s)} \E_{b \sim \pi_y(\cdot | s)} \left[\tilde d_\rho^{\pi_x, \pi_y}(s)  (\grad_x \log \pi_x(a | s)) Q^{\pi_x, \pi_y}(s, a, b)\right].
  \end{align*}
  Thus, for any $s \in \MS, a \in \MA$, we have
  $$
\frac{\partial V_\rho(x,y)}{\partial x_{s,a}} = (1 - \gdx) \tilde d_\rho^{\pi_x, \pi_y}(s) \E_{b \sim \pi_y(\cdot | s)} \left[ Q^{\pi_x, \pi_y}(s,a,b) \right],
$$
and so it follows that
$$
\left| \frac{\partial V_\rho(x,y)}{\partial x_{s,a}} \right| \leq \frac{d_\rho^{\pi_x, \pi_y}(s)}{\zeta} \left| \E_{b \sim \pi_y(\cdot | s)} \left[ Q^{\pi_x, \pi_y}(s,a,b) \right]\right|.
$$
The inequality for the derivative with respect to $y$ follows in a symmetric manner.
\end{proof}

\begin{lemma}[Performance difference lemma]
  \label{lem:perf_diff}
  For all policies $\pi_1, \pi_1', \pi_2, \pi_2'$ and distributions $\rho \in \Delta(\MS)$,
  \begin{align*}
    V_\rho(\pi_1, \pi_2) - V_\rho(\pi_1', \pi_2) &= \sum_{s \in \MS} \tilde d_\rho^{\pi_1, \pi_2}(s) \E_{a \sim \pi_1(\cdot | s)} \E_{b \sim \pi_2(\cdot | s)} \left[ A^{\pi_1',\pi_2}(s,a,b) \right] \\
    V_\rho(\pi_1, \pi_2) - V_\rho(\pi_1, \pi_2') &= \sum_{s \in \MS} \tilde d_\rho^{\pi_1, \pi_2}(s) \E_{a \sim \pi_1(\cdot | s)} \E_{b \sim \pi_2(\cdot | s)} \left[ A^{\pi_1',\pi_2}(s,a,b) \right].
  \end{align*}
\end{lemma}
\begin{proof}[\pfref{lem:perf_diff}]
  Note that, for any $s \in \MS$, 
  \begin{align*}
    & V_{s}(\pi_1, \pi_2) - V_{s}(\pi_1', \pi_2) \\
    &= \E_{\tau \sim \pr^{\pi_1, \pi_2}(\cdot | {s})} \left[ \sum_{t=0}^T R(s_t, a_t, b_t) \right] - V_{s}(\pi_1', \pi_2) \\
    &= \E_{\tau \sim \pr^{\pi_1, \pi_2}(\cdot | {s})} \left[ \sum_{t=0}^T R(s_t, a_t, b_t)  + V_{s_t}(\pi_1', \pi_2) - V_{s_t}(\pi_1', \pi_2)\right] - V_{s}(\pi_1', \pi_2) \\
    &= \E_{\tau \sim \pr^{\pi_1, \pi_2}(\cdot | {s})} \left[ \sum_{t=0}^T R(s_t, a_t, b_t)  + \indic_{t+1 \leq T} V_{s_{t+1}}(\pi_1', \pi_2) - V_{s_t}(\pi_1', \pi_2)\right] \\
    &= \E_{\tau \sim \pr^{\pi_1, \pi_2}(\cdot | {s})} \left[ \sum_{t=0}^T R(s_t, a_t, b_t)  +\E\left[ \indic_{t+1 \leq T} V_{s_{t+1}}(\pi_1', \pi_2) | s_t, a_t, b_t \right]- V_{s_t}(\pi_1', \pi_2)\right] \\
    &= \E_{\tau \sim \pr^{\pi_1, \pi_2}(\cdot | {s})} \left[ \sum_{t=0}^T A^{\pi_1', \pi_2}(s_t, a_t, b_t) \right] \\
    &= \sum_{s' \in \MS} \tilde d_{s}^{\pi_1, \pi_2}(s') \E_{a \sim \pi_1(\cdot | s')} \E_{b \sim \pi_2(\cdot | s')} \left[ A^{\pi_1', \pi_2}(s', a, b)\right].
  \end{align*}
The proof of the second inequality in the lemma is symmetric.
\end{proof}

\begin{lemmod}{lem:gd_greedy}{a}
  \label{lem:gd_greedy_full}
    Suppose that players follow the $\veps$-greedy direct
  parameterization of \pref{ass:greedy} with parameters $\veps_x$ and
  $\veps_y$. Then for all $x\in\xset$, $y\in\yset$ we have 
  \begin{align}
      \Vrho(x,y) - \min_{x'} \Vrho(x',y)
    &  \leq{}
      \min_{\pi_1 \in \Pi_1^*(\pi_y)}
\nrm*{\frac{\drho^{\pi_1,\pi_y}}{\rho}}_{\infty}\prn*{    \frac{1}{\zeta}\max_{\bar{x}\in\Delta(\cA)^{\abs*{\cS}}}\tri*{\grad_{x}\Vrho(x,y),x-\bar{x}}
    + \frac{2\veps_x}{\zeta^3}},
    \label{eq:gd-appendix}
  \end{align}
  and
  \begin{align}
    \max_{y'} \Vrho(x,y') - \Vrho(x,y)
  &  \leq{} \min_{\pi_2 \in \Pi_2^*(\pi_x)}
\nrm*{\frac{\drho^{\pi_x,\pi_2}}{\rho}}_{\infty}\prn*{    \frac{1}{\zeta}\max_{\bar{y}\in\Delta(\cB)^{\abs*{\cS}}}\tri*{\grad_{y}\Vrho(x,y),\bar{y}-y}
    + \frac{2\veps_y}{\zeta^3}},
  \end{align}
\end{lemmod}

\begin{proof}[\pfref{lem:gd_greedy}]
We prove the inequality for $x$ player. The inequality for the $y$
player follows by symmetry. %
For a policy $\pi_y$, let $\pi_1^*(\pi_y) \in \Pi_1^*(\pi_y)$ denote a policy minimizing $\nrm*{\frac{d_\rho^{\pi_1,\pi_y}}{\rho}}_\infty$ (whose existence follows from compactness of the space of policies).

Using the performance difference lemma, we have 
\begin{align*}
&   \Vrho(x,y) - \min_{x'}\Vrho(x',y) \\
  &\leq{} \Vrho(\pi_x, \pi_y) - \Vrho(\pi_1^*(\pi_y), \pi_y) \\
&  =                            \sum_{s,a}\tilde\drho^{\pistar_1(\pi_y),\pi_y}(s)\pistar_1(\pi_y)(a\mid{}s)\En_{b\sim{}\pi_y(\cdot\mid{}s)}\brk*{-A^{\pi_x,\pi_y}(s,a,b)}\\
  &  \leq{}    \sum_{s}\tilde\drho^{\pistar_1(\pi_y),\pi_y}(s)\max_{a}\En_{b\sim{}\pi_y(\cdot\mid{}s)}\brk*{-A^{\pi_x,\pi_y}(s,a,b)}\\
  &  \leq{} \nrm*{\frac{\tilde\drho^{\pistar_1(\pi_y),\pi_y}}{\tilde\drho^{\pi_x,\pi_y}(s)}}_{\infty}\sum_{s}\tilde\drho^{\pi_x,\pi_y}(s)\max_{a}\En_{b\sim{}\pi_y(\cdot\mid{}s)}\brk*{-A^{\pi_x,\pi_y}(s,a,b)}.
\end{align*}
    We observe that
$\nrm*{\frac{\tilde\drho^{\pistar_1(\pi_y),\pi_y}}{\tilde\drho^{\pi_x,\pi_y}}}_{\infty}\leq{}\frac{1}{\zeta}\nrm*{\frac{\drho^{\pistar_1(\pi_y),\pi_y}}{\rho}}_{\infty}\leq{}\frac{1}{\zeta}C_{\cG}$.
Next,
we have
\begin{align}
  &  \sum_{s,a}\tilde\drho^{\pi_x,\pi_y}(s)\max_{a}\En_{b\sim{}\pi_y(\cdot\mid{}s)}\brk*{-A^{\pi_x,\pi_y}(s,a,b)}\nonumber\\
  &=\max_{\bar{x}\in\Delta(\cA)^{\abs*{\cS}}}\sum_{s,a}\tilde\drho^{\pi_x,\pi_y}(s)\bar{x}_{s,a}\En_{b\sim{}\pi_y(\cdot\mid{}s)}\brk*{-A^{\pi_x,\pi_y}(s,a,b)}\nonumber\\
  &=\max_{\bar{x}\in\Delta(\cA)^{\abs*{\cS}}}\sum_{s,a}\tilde\drho^{\pi_x,\pi_y}(s)(\pi_x(a\mid{}s)-\bar{x}_{s,a})\En_{b\sim{}\pi_y(\cdot\mid{}s)}\brk*{Q^{\pi_x,\pi_y}(s,a,b)}\nonumber\\
  &=\max_{\bar{x}\in\Delta(\cA)^{\abs*{\cS}}}\sum_{s,a}\tilde\drho^{\pi_x,\pi_y}(s)((1-\veps_x)x_{s,a}
    +
    \veps_xA^{-1}
    -\bar{x}_{s,a})\En_{b\sim{}\pi_y(\cdot\mid{}s)}\brk*{Q^{\pi_x,\pi_y}(s,a,b)},\nonumber\\
  &\leq{}\max_{\bar{x}\in\Delta(\cA)^{\abs*{\cS}}}\sum_{s,a}\tilde\drho^{\pi_x,\pi_y}(s)((1-\veps_x)x_{s,a}
                                                                                                                                                            +
    \veps_xA^{-1}-\veps_x\bar{x}_{s,a}
    -(1-\veps_x)\bar{x}_{s,a})\En_{b\sim{}\pi_y(\cdot\mid{}s)}\brk*{Q^{\pi_x,\pi_y}(s,a,b)}\nonumber\\
  &\leq{}(1-\veps_x)\max_{\bar{x}\in\Delta(\cA)^{\abs*{\cS}}}\sum_{s,a}\tilde\drho^{\pi_x,\pi_y}(s)(x_{s,a}
    -\bar{x}_{s,a})\En_{b\sim{}\pi_y(\cdot\mid{}s)}\brk*{Q^{\pi_x,\pi_y}(s,a,b)}
    + \frac{2\veps_x}{\zeta^2}\nonumber\\
  \label{eq:gd-last}
  &=
    \max_{\bar{x}\in\Delta(\cA)^{\abs*{\cS}}}\tri*{\grad_{x}\Vrho(x,y),x-\bar{x}}
    + \frac{2\veps_x}{\zeta^2},
\end{align}
where \pref{eq:gd-last} follows from \pref{prop:gradients}.
Rearranging, this establishes that
\begin{align*}
  \Vrho(x,y) - \Vrho(\xstar(y),y)
  &  \leq{}
\nrm*{\frac{\drho^{\pistar_1(\pi_y),\pi_y}}{\rho}}_{\infty}\prn*{    \frac{1}{\zeta}\max_{\bar{x}\in\Delta(\cA)^{\abs*{\cS}}}\tri*{\grad_{x}\Vrho(x,y),x-\bar{x}}
    + \frac{2\veps_x}{\zeta^3}}.
\end{align*}

\end{proof}

The following lemma, which is a consequence of Lemma E.3 of
\citet{agarwal2019optimality}, establishes that the direct
parameterization leads to Lipschitz gradients.
\begin{lemma}[Smoothness]
  \label{lem:mg-smoothness}
  For all starting states $s_0$, and for all policies $x, x', y,y'$, it holds that
  \begin{align*}
\left\| \grad_x V_{s_0}(x,y) - \grad_x V_{s_0}(x', y') \right\|_2 \leq \frac{4(1-\zeta) A}{\zeta^3} \| (x,y) - (x',y') \|_2, \\
\left\| \grad_y V_{s_0}(x,y) - \grad_y V_{s_0}(x', y') \right\|_2 \leq \frac{4(1-\zeta) A}{\zeta^3} \| (x,y) - (x',y') \|_2. \\
  \end{align*}
\end{lemma}

\section{Two-Timescale SGDA}
\label{app:sgda}

\subsection{Algorithm and Main Theorem}
 Throughout this section we will consider compact and convex subsets
 $\MX \subset \BR^{d_x}, \MY \subset \BR^{d_y}$ of
 Euclidean space. Our goal will be to find approximate equilibria for
 the game
 \[
\min_{x\in\cX}\max_{y\in\cY}f(x,y),
 \]
where $f : \MX \times \MY \ra \BR$ is a continuously differentiable
function. We assume that we can only access $f$ through a
\emph{stochastic first-order oracle} (\pref{ass:oracle}), and we
analyze a two-timescale version of simultaneous gradient
descent-ascent (SGDA) in this model. Before stating the algorithm, we
state our regularity assumptions on the function $f$ and the oracle.

Define the {\it max} function $\Phi : \MX \ra \BR$ and the {\it min} function $\Psi : \MY \ra \BR$ as follows:
 $$
\Phi(x) := \max_{y \in \MX} f(x,y), \qquad \Psi(y) := \min_{x \in \MX} f(x,y).
$$
Moreover, let $D_\MX$ denote the diameter of $\MX$ and $D_\MY$ denote the diameter of $\MY$. 
 We make the following
 assumptions about $f(x,y)$. To state the assumption, let
 $\ystar(x)\in\argmax_{y\in\cY}f(x,y)$ and
 $\xstar(y)\in\argmin_{x\in\cX}f(x,y)$ denote arbitrary best-response
 functions for the $y$ and $x$ players, respectively.

\begin{assumption}
  \label{asm:lip-smooth-f}
  Assume that $\MX \subset \BR^{d_\sx}, \MY \subset \BR^{d_\sy}$ are closed and compact subsets of Euclidean space and $f : \MX \times \MY \ra \BR$. We assume that $f$ satisfies: %
  \begin{enumerate}
  \item $f$ is $\ell$-smooth and $L$-Lipschitz. %
    \label{it:f-smooth-lip}
      \item For some constants $\gdy \geq 0, \mu_\sy > 0$, for each $x \in \MX$, the function $y \mapsto f(x,y)$
        satisfies the following gradient domination condition:
        \[
          \max_{\bar y \in \cY: \nrm*{\bar{y}-y}\leq{}1} \lng\bar{y}-y, \grad_y f(x,y) \rng \geq \mu_\sy \cdot (f(x,y^*(x)) - f(x, y)) - \gdy.
        \]
        \label{it:y-pl}
  \item For some constants $\gdx \geq 0, \mu_\sx > 0$, for each $y \in \MX$, the function $x \mapsto f(x,y)$ satisfies the following gradient domination condition:
    $$
\max_{\bar x \in \MX, \| \bar x - x \| \leq 1} \lng x- \bar x, \grad_x f(x,y) \rng \geq \mu_\sx \cdot (f(x,y) - f(x^*(y), y)) - \gdx.
    $$ \label{it:x-pl}
\end{enumerate}
\end{assumption}
\begin{remark}[Empty interior]
  \label{rem:empty-interior}
  If the interior of $\MX \times \MY \subset \BR^{d_x + d_y}$, denoted by $(\MX \times \MY)^\circ$, is empty (which is the case for the direct parametrization of policies in Markov games), then in order to ensure that $\grad f(x,y)$ is well-defined on $\MX \times \MY$, we make the technical assumption that $f$ is continuously differentiable on a closed neighborhood $\tilde \MX \times \tilde \MY$, where $\MX \subset \subset \tilde \MX, \MY \subset\subset \tilde \MY$\footnote{The notation $\MX \subset\subset \tilde \MX$ means that $\MX$ is {\it compactly contained} in $\tilde \MX$, i.e., $\MX \subset \tilde \MX^\circ$.}, which may be assumed without loss of generality to be convex. It is straightforward to check that this assumption holds in our application to Markov games.
  
Suppose $f$ satisfies \pref{asm:lip-smooth-f}. In the event that the interior $(\MX \times \MY)^\circ$ is empty, by compactness of $\MX$ and $\MY$, for any $\delta > 0$, there are closed convex neighborhoods $\MX_\delta, \MY_\delta$ with $\MX \subset\subset \MX_\delta, \MY \subset\subset \MY_\delta$, so that any point in $\MX_\delta \times \MY_\delta$ is at most distance $\delta$ from a point in $\MX \times \MY$, $f$ is $(\ell + \delta)$-smooth and $(L+\delta)$-Lipschitz on $\MX_\delta \times \MY_\delta$, and items \ref{it:y-pl} and \ref{it:x-pl} hold for any $x \in \MX_\delta, y \in \MY_\delta$ with constants $\gdy + \delta$, $\mu_{\sy} - \delta$, $\mu_{\sx} - \delta$, and $\gdx + \delta$. We will use this fact in the proof of \pref{lem:kl-moreau}.
\end{remark}

We formalize the stochastic first-order oracle model our algorithm
works in as follows. In this section only, we will denote the iterates of stochastic gradient descent-ascent using $x_t, y_t$ (as opposed to previous sections where we wrote $x\^i, y\^i$).

Given a random variable $\xi \in \Xi$ with law
$\p$ (for some sample space $\Xi$), a {\it stochastic first-order oracle} $G : \MX
\times \MY \times \Xi \ra \BR^{d_x + d_y}$ satisfies the following properties.
\begin{assumption}[Stochastic first-order oracle]
  \label{ass:oracle}
For variance parameters $\sigma_x,\sigma_y > 0$, the stochastic oracle $G(x,y,\xi) = (G_x(x,y,\xi), G_y(x,y,\xi))$ satisfies:
  \begin{align*}
    &\E[G(x,y,\xi)] = \grad f(x,y), \\
    &\E[ \| G_x(x,y,\xi) - \grad_x f(x,y,\xi) \|^2] \leq \sigma_\sx^2, \\
    &\E[ \| G_y(x,y,\xi) - \grad_y f(x,y,\xi) \|^2]  \leq \sigma_\sy^2.
  \end{align*}
\end{assumption}
Given the stochastic first-order oracle $G = (G_x, G_y)$, the {\it two-timescale stochastic simultaneous GDA algorithm} (or {\it SGDA}) draws a sample $\xi_{t-1} \sim \p$, and performs the updates
\begin{align}
  x_t &\gets \proj_{\MX} \left(x_{t-1} - \eta_\sx  G_\sx(x_{t-1}, y_{t-1},\xi_{t-1})\right), \label{eq:sgda-x}\\
  y_t &\gets \proj_{\MY} \left( y_{t-1} +  \eta_\sy G_\sy(x_{t-1}, y_{t-1}, \xi_{t-1} \right).\label{eq:sgda-y}
\end{align}
\paragraph{Main theorem.}
Our main theorem for SGDA, \pref{thm:sgda} (the full version of
\pref{thm:sgda_body}), shows that if the learning rate $\eta_\sx$ of
two-timescale SGDA is chosen sufficiently small relative to $\eta_\sy$, the iterates $x_t$ will approach, on average, the optimal point $x^*$.

For simplicity of presentation, we make the following assumptions
regarding the various parameters: $\min\{L, \ell, \sigma_\sx,
\sigma_\sy, 1/\mu_\sx, 1/\mu_\sy\} \geq 1$. These assumptions are
essentially without loss of generality (at the cost of potentially
worse bounds), since $L,\ell,\sigma_\sx, \sigma_\sy,1/\mu_\sx,
1/\mu_\sy$ are {\it upper bounds} on various properties of the
function $f$ and the gradient oracle $G$. Finally, let $\Phi_{1/2\ell}$ denote the {\it Moreau envelope} of $\Phi$ with parameter $1/2\ell$ (see \pref{app:technical-prelim}).
\begin{thmmod}{thm:sgda_body}{a}
  \label{thm:sgda}
  Suppose that \pref{asm:lip-smooth-f} and \pref{ass:oracle} hold.
For any $\ep \in (0,1)$, for two-timescale SGDA with $\eta_\sy = \Theta\left(\frac{\ep^4\mu_\sy^2}{\ell^3 (L^2 + \sigma_\sy^2)(L/\ell+1)^2}\right), \eta_\sx = \Theta \left(\frac{\ep^8\mu_\sy^4}{\ell^5 L (L/\ell+1)^4 (L^2 + \sigma_\sy^2) \sqrt{L^2 + \sigma_\sx^2}} \wedge \frac{\ep^2}{\ell(L^2 + \sigma_\sx^2)} \right)$, we have
  \begin{align*}
    &\frac{1}{T+1}\sum_{t=0}^{T}\En\| \nabla \Phi_{1/2\ell}(x_t) \|   \leq \ep + \sqrt{\frac{8\ell \gdy}{\mu_\sy}},\\
    \intertext{and}
    &\frac{1}{T+1}\sum_{t=0}^{T}\En\brk*{\Phi(x_t)} - \Phi(\xstar)  \leq{} \prn*{\frac{1}{\mu_\sx}+ \frac{L}{2\ell}}\cdot{}\left(\eps + \sqrt{\frac{8\ell \gdy}{\mu_\sy}}\right) + \frac{\gdx}{\mu_\sx},
  \end{align*}
  for $T \geq \Omega\left(\frac{(D_\MX + D_\MY) L}{\ep^2 \eta_\sx} \right)$. 
\end{thmmod}
To interpret the parameter settings in \pref{thm:sgda}, note that if $\gdx = \gdy =
0$ and $\sigma_\sx, \sigma_\sy, L, \ell, \mu_\sx, \mu_\sy, D_\MX$, and 
$D_\MY$ are all viewed as constants, then if we set $\eta_\sy
\asymp\ep^4, \eta_\sx \asymp\ep^8$, we are guaranteed to find an $\ep$-suboptimal point within $T \asymp \ep^{-10}$ iterations.

\subsection{Technical Preliminaries for Proof}
\label{app:technical-prelim}
\paragraph{Non-smooth minimization in the constrained setting.} A function $\varphi : \MX \ra \BR$ is defined to be {\it $\ell$-weakly convex} if $x \mapsto \varphi(x) + \frac{\ell}{2} \| x \|^2$ is convex. In such a case, we may extend $\varphi$ to a function $\varphi : \BR^{d_x} \ra \BR \cup \{ \infty\}$, by $\varphi(x) =\infty$ for $x \not \in \MX$, and the extended function $\varphi$ remains $\ell$-weakly convex. For a $\ell$-weakly convex function $\varphi$ and $x \in \BR^{d_\sx}$, the {\it subgradient} of $\varphi$ at $x$ may be defined in terms of the subgradient of the convex function $\tilde \varphi(x) := \varphi(x) + \frac{\ell}{2} \| x \|^2$:
\begin{equation}
  \label{eq:subgradient-wc}
\partial \varphi(x) := \partial \tilde \varphi(x) - \ell x.
\end{equation}
For any $\lambda > 0$, the {\it Moreau envelope} $\varphi_\lambda : \BR^{d_x} \ra \BR$ and {\it proximal map} $\prox_{\lambda \varphi} : \BR^{d_x} \ra \MX$ of $\varphi$ are defined, respectively as follows \citep{davis_stochastic_2018-1}:
\begin{align}
  \varphi_\lambda(x) &:= \min_{x' \in \MX} \left\{ \varphi(x') + \frac{1}{2\lambda} \| x' - x \|^2 \right\} \label{eq:moreau_def}\\
  \prox_{\lambda \varphi}(x) &:= \argmin_{x' \in \MX} \left\{ \varphi(x') + \frac{1}{2\lambda} \| x' - x \|^2 \right\}.\label{eq:prox_def}
\end{align}
Let $\Phi(x) = \max_{y \in \MY} f(x,y)$ and $x^* \in\arg \min_{x \in \MX} \Phi(x)$. %
\begin{lemma}[\citet{lin2019gradient}]
  \label{lem:phi-facts}
  Suppose $f : \MX \times \MY \ra \BR$ is $L$-Lipschitz and $\ell$-smooth. Then:
  \begin{enumerate}
  \item $\Phi(x)$ is $L$-Lipschitz.    \label{it:phi-lip}
  \item $\Phi(x)$ is $\ell$-weakly convex.     \label{it:phi-wc}
  \end{enumerate}
\end{lemma}

\begin{lemma}[\citet{davis_stochastic_2018-1}]
  \label{lem:davis-moreau-2018}
  Suppose $\varphi : \MX \ra \BR$ is $\ell$-weakly convex. Then:
  \begin{enumerate}
  \item $\grad \varphi_{1/2\ell} (x) = 2\ell(x - \prox_{\varphi/(2\ell)}(x))$. \label{it:phi-grad}
  \item If $\| \grad \varphi_{1/2\ell}(x) \|_2 \leq \ep$, then there is $\hat x \in \MX$ so that $\| x - \hat x \|\leq \ep / (2\ell)$ and $\min_{\xi \in \partial \varphi(\hat x)} \| \xi \| \leq \ep$.    \label{it:phi-moreau}
    \item $\grad \varphi_{1/2\ell}(\cdot)$ is $\ell$-Lipschitz. \label{it:moreau-smooth}
  \end{enumerate}
\end{lemma}

The following theorem establishes some fundamental properties of $\Phi(x)$.
\begin{theorem}[Danskin's theorem]
  \label{thm:danskin}
  Suppose $\MX \subset \BR^{d_\sx}$ is an open subset, $\MY \subset \BR^{d_\sy}$ is compact, and $f : \MX \times \MY \ra \BR$ is continuously differentiable and $\ell$-weakly convex. Then $\Phi(x) := \max_{y \in \MY} f(x,y)$ is $\ell$-weakly convex and
  $$
\partial \Phi(x) = \conv \left\{ \grad_\sx f(x,y) : y \in Y(x)\right\},
  $$
where
$$
Y(x) \ldef \left\{ y : f(x,y) = \max_{y' \in \MY} f(x,y) \right\}.
$$
\end{theorem}

\paragraph{Descent lemmas for two-timescale SGDA.}
Let $(x_t, y_t)$ denote the iterates of two-timescale SGDA, as in
\pref{eq:sgda-x} and \pref{eq:sgda-y}. Define $\Delta_t := \Phi(x_t) - f(x_t, y_t)$. 

The following lemma, whose proof relies on item \ref{it:phi-grad} of \pref{lem:davis-moreau-2018} was shown in \cite{lin2019gradient}; technically, the proof there was given for the unconstrained case (namely, $\MX = \BR^{d_x}$) and the case where $y \mapsto f(x,y)$ is concave for each $x$, but the proof holds with minimal modifications to our case. For completeness we give the full proof. 
\begin{lemma}[\citet{lin2019gradient}, Lemma D.3]
  \label{lem:lin-moreau}
For two-timescale SGDA, we have:
$$
\E [ \Phi_{1/(2\ell)}(x_t)] \leq \E[\Phi_{1/(2\ell)}(x_{t-1})] + 2 \eta_x \ell \E[\Delta_{t-1}] - \frac{\eta_x}{4} \E \left[ \| \grad \Phi_{1/(2\ell)} (x_{t-1}) \|^2 \right] + \eta_x^2 \ell (L^2 + \sigma_\sx^2).
$$
\end{lemma}
\begin{proof}[Proof of \pref{lem:lin-moreau}]
  Set $\hat x_{t-1} := \prox_{\Phi/2\ell}(x_{t-1})$, so that
  \begin{equation}
    \label{eq:phi-x-hat}
\Phi_{1/2\ell}(x_{t}) \leq \Phi(\hat x_{t-1}) + \ell \| \hat x_{t-1} - x_t \|^2 \leq \Phi_{1/2\ell}(x_{t-1}) + \ell \| \hat x_{t-1} - x_t \|^2 - \ell \| \hat x_{t-1} - x_{t-1} \|^2.
    \end{equation}

  Since $\hat x_{t-1} \in \MX$ and $x_t = \proj_\MX(x_{t-1} - \eta_\sx  G_\sx(x_{t-1}, y_{t-1}, \xi_{t-1}))$, we have
  \begin{align*}
    \| \hat x_{t-1} - x_t \|^2 & \leq\left\| \hat x_{t-1} - \left( x_{t-1} - \eta_\sx  G_\sx(x_{t-1}, y_{t-1}, \xi_{t-1})\right) \right\|^2 \\
                               & \leq \| \hat x_{t-1} - x_{t-1} \|^2 + \|\eta_\sx G_\sx(x_{t-1}, y_{t-1}, \xi_{t-1}) \|^2 + 2 \lng \hat x_{t-1}-x_{t-1}, \eta_\sx G_\sx(x_{t-1}, y_{t-1}, \xi_{t-1})\rng .
  \end{align*}
  Taking the expectation of both sides gives
  \begin{align}
    & \E[\| \hat x_{t-1} - x_t \|^2] \nonumber\\
    & \leq \E [ \| \hat x_{t-1} - x_{t-1} \|^2] + \eta_\sx^2\E[\| G_\sx(x_{t-1}, y_{t-1}, \xi_{t-1}) \|^2] + 2 \lng \hat x_{t-1}- x_{t-1}, \eta_\sx\grad_\sx f(x_{t-1}, y_{t-1}) \rng\nonumber\\
    \label{eq:moreau-diff-ub}
    & \leq \E [\| \hat x_{t-1} - x_{t-1} \|^2] + \eta_\sx^2(L^2 + \sigma^2) + 2 \E [\lng \hat x_{t-1} - x_{t-1}, \eta_\sx\grad_\sx f(x_{t-1}, y_{t-1}) \rng].
  \end{align}
  Next, we observe that
  \begin{align}
    & \lng \hat x_{t-1} - x_{t-1}, \grad_\sx f(x_{t-1}, y_{t-1}) \rng \nonumber\\
    & \leq f(\hat x_{t-1}, y_{t-1}) - f(x_{t-1}, y_{t-1}) + \frac{\ell}{2} \| \hat x_{t-1} - x_{t-1} \|^2\nonumber\\
    & \leq \Phi(\hat x_{t-1}) - f(x_{t-1}, y_{t-1}) + \frac{\ell}{2} \| \hat x_{t-1} - x_{t-1} \|^2 \nonumber\\
    & = \Phi(\hat x_{t-1}) + \Delta_{t-1} - \Phi(x_{t-1}) + \frac{\ell}{2} \| \hat x_{t-1} - x_{t-1} \|^2\nonumber\\
    \label{eq:grad-ip-ub}
    & \leq \Delta_{t-1} - \frac{\ell}{2} \| \hat x_{t-1} - x_{t-1} \|^2 \leq \Delta_{t-1},
  \end{align}
  where the first inequality above follows since $f$ is $\ell$-smooth, the second inequality follows since $\Phi(\hat x_{t-1}) \geq f(\hat x_{t-1}, y_{t-1})$, and the final inequality \pref{eq:grad-ip-ub}) follows since $\Phi(\hat x_{t-1}) + \ell \| \hat x_{t-1} - x_{t-1} \|^2 \leq \Phi(x_{t-1})$ by definition of $\prox_{\Phi/2\ell}(\cdot)$.

  By equations \pref{eq:phi-x-hat}, \pref{eq:moreau-diff-ub}, and \pref{eq:grad-ip-ub}, we get
  \begin{align}
     \E[\Phi_{1/2\ell}(x_t)] 
    & \leq \E[\Phi_{1/2\ell}(x_{t-1})] + \ell \eta_\sx^2(L^2 + \sigma^2) + 2\ell \eta_\sx \lng \hat x_{t-1} - x_{t-1}, \grad_\sx f(x_{t-1}, y_{t-1})\rng  \nonumber\\
    & \leq \E[\Phi_{1/2\ell}(x_{t-1})] + 2\eta_\sx \ell \E[\Delta_{t-1}] - \eta_\sx \ell^2 \E [ \| \hat x_{t-1} - x_{t-1} \|^2] + \ell \eta_\sx^2 (L^2 + \sigma^2) \nonumber\\
    & \leq  \E[\Phi_{1/2\ell}(x_{t-1})] + 2\eta_\sx \ell \E[\Delta_{t-1}] - \frac{\eta_\sx}{4}\E [ \| \grad \Phi_{1/2\ell}(x_{t-1})\|^2] + \ell \eta_\sx^2 (L^2 + \sigma^2) \nonumber.
  \end{align}
\end{proof}
To show that the $y$ player approximately tracks the best response (in
terms of value), we make use of a slightly different potential function. To describe the approach, set $\lambda \in (0,1/\ell)$ to be specified later. Letting $(x_t, y_t)$ be the iterates of two-timescale SGDA, for each $t \geq 0$, let $\phi_{t-1} : \MY \ra \BR$ be the function $\phi_{t-1}(y) := -f(x_{t-1}, y)$, and set $\psi_{t,\lambda}(y) := - (\phi_{t-1})_\lambda(y)$ to be the negated Moreau envelope of $\phi_{t-1}$ with parameter $\lambda$. Our first lemma states that $\psi_{t,\lambda}$ does not change much from iteration to iteration.
\begin{lemma}
  \label{lem:psi-slow}
  For all $t \geq 1$, and $y \in \MY$, we have
  $$
| \psi_{t,\lambda}(y) - \psi_{t-1,\lambda}(y) | \leq L \cdot \| x_{t-1} - x_t \|.
  $$
\end{lemma}
\begin{proof}[\pfref{lem:psi-slow}]
  Note that for any $y \in \MY$,
  \begin{align*}
| \psi_{t,\lambda}(y) - \psi_{t-1,\lambda}(y)| & = \left| \min_{y' \in \MY} \left\{\frac{1}{2\lambda} \| y - y' \|^2 - f(x_t, y')\right\} - \min_{y' \in \MY} \left\{\frac{1}{2\lambda} \| y - y' \|^2 - f(x_{t-1}, y')\right\} \right|.
  \end{align*}
  Since for all $y' \in \MY$ we have
  $$
\left| \frac{1}{2\lambda} \| y - y' \|^2 - f(x_t, y') - \left(\frac{1}{2\lambda} \| y - y' \|^2 - f(x_{t-1}, y') \right)\right\| \leq L \| x_{t-1} - x_t \|,
    $$
    the conclusion follows.
\end{proof}

Now let $\Gamma_t := \| \grad \psi_{t,\lambda}(y_t) \|$. The following
lemma shows that as long as $\Gamma_t$ stays large, $\psi_{t,\lambda}$
decreases each iteration (up to an error term controlled by the
learning rate of the $x$ player).
\begin{lemma}
  \label{lem:deltat-dec-stoch}
For two-timescale SGDA, for all $t \geq 0$, as long as $\eta_\sy \leq 1/(2\ell)$ and $\lambda \in (0,1/\ell)$,
\begin{equation}
  \label{eq:deltat-dec-stoch}
  \E[\psi_{t,\lambda}(y_t) | \MF_{t-1}] \geq \psi_{t-1,\lambda}(y_{t-1}) + \eta_\sy \lambda(1/\lambda - \ell) \cdot \Gamma_{t-1}^2 - L \eta_\sx \sqrt{L^2 + \sigma_\sx^2} - \frac{\eta_\sy^2 (L^2 + \sigma_\sy^2)}{2\lambda}.
  \end{equation}
\end{lemma}
\begin{proof}[\pfref{lem:deltat-dec-stoch}]
  Write $g_{t-1} = G_y(x_{t-1}, y_{t-1}, \xi_{t-1})$. Set
  $$
\hat y_{t-1} := \prox_{\lambda \phi_{t-1}}(y_{t-1}) = \argmin_{y' \in \MY} \left\{ \frac{1}{2\lambda} \| y_{t-1} - y' \|^2 - f(x_{t-1}, y') \right\}.
$$
We next need the following lower bound on $\psi_{t-1,\lambda}(y_t)$ in
terms of $\psi_{t-1,\lambda}(y_{t-1})$; this calculation was carried
out in \citet[Eqs. (2.4) -- (2.6)]{davis_stochastic_2018-1}, but we
prove the following self-contained lemma after the conclusion of this proof for completeness.
\begin{lemma}[\citet{davis_stochastic_2018-1}]
  \label{lem:dd-trick}
  For $\lambda \in (0,1/\ell)$, we have
  $$
  \E \left[ \psi_{t-1,\lambda}(y_t)  | \MF_{t-1}\right] \geq \psi_{t-1,\lambda}(y_{t-1}) + \frac{\eta_\sy}{\lambda} \left( f(x_{t-1}, \hat y_{t-1}) - f(x_{t-1}, y_{t-1}) - \frac{1}{2\lambda} \| y_{t-1} - \hat y_{t-1} \|^2 \right) - \frac{\eta_\sy^2 (L^2 + \sigma_\sy^2)}{2\lambda}.
$$
\end{lemma}
By \pref{lem:dd-trick}, we have
  \begin{align*}
    \E \left[ \psi_{t-1,\lambda}(y_t)  | \MF_{t-1}\right] &\geq \psi_{t-1,\lambda}(y_{t-1}) + \frac{\eta_\sy}{\lambda} \left( f(x_{t-1}, \hat y_{t-1}) - f(x_{t-1}, y_{t-1}) - \frac{1}{2\lambda} \| y_{t-1} - \hat y_{t-1} \|^2 \right) - \frac{\eta_\sy^2 (L^2 + \sigma_\sy^2)}{2\lambda}\\
                                                          & \geq \psi_{t-1,\lambda}(y_{t-1}) + \frac{\eta_\sy}{\lambda} \cdot \left(\lambda^2 (1/\lambda - \ell) \cdot \| \grad \psi_{t-1,\lambda}(y_{t-1}) \|^2\right) - \frac{\eta_\sy^2 (L^2 + \sigma_\sy^2)}{2\lambda} \\
                                                          & = \psi_{t-1,\lambda}(y_{t-1}) + \eta_\sy \lambda (1/\lambda - \ell) \cdot \Gamma_{t-1}^2 - \frac{\eta_\sy^2 (L^2 + \sigma_\sy^2)}{2\lambda},
  \end{align*}
  where the second inequality above follows by $\ell$-smoothness of $f$. 
  By \pref{lem:psi-slow}, we have
  \begin{align*}
    \E[\psi_{t-1,\lambda}(y_t) - \psi_{t,\lambda}(y_t) | \MF_{t-1}] & \leq L \cdot \E [\| x_{t-1} - x_t \| | \MF_{t-1}] \\
                                                                    & \leq L \cdot \E [\eta_\sx \cdot \| G_x(x_{t-1}, y_{t-1}, \xi_{t-1}) \| | \MF_{t-1})  ] \\
                                                                    & \leq L \eta_\sx \cdot \sqrt{L^2 + \sigma_\sx^2}.
  \end{align*}
  Combining the above displays gives that
  $$
\E[\psi_{t,\lambda}(y_t) | \MF_{t-1}] \geq \psi_{t-1,\lambda}(y_{t-1}) + \eta_\sy \lambda(1/\lambda - \ell) \cdot \Gamma_{t-1}^2 - L \eta_\sx \sqrt{L^2 + \sigma_\sx^2} - \frac{\eta_\sy^2 (L^2 + \sigma_\sy^2)}{2\lambda}.
  $$

\end{proof}
\begin{proof}[Proof of \pref{lem:dd-trick}]
  The proof is exactly the argument in \citet[Eqs. (2.4) -- (2.6)]{davis_stochastic_2018-1} and similar to that used in the proof of \pref{lem:lin-moreau}, but for completeness we repeat this calculation using our notation. In the setting of \pref{lem:deltat-dec-stoch}, set $\hat y_{t-1} := \prox_{-\lambda \cdot\phi_{t-1}}(y_{t-1})$, and $g_{t-1} = G_\sy(x_{t-1}, y_{t-1}, \xi_{t-1})$. Then
  \begin{align}
    & \E\left[ -(\phi_{t-1})_\lambda(y_t)  | \MF_{t-1} \right] \nonumber\\
    \label{eq:dd-1}
    & \leq \E \left[ -\phi_{t-1}(\hat y_{t-1}) + \frac{1}{2\lambda} \| y_t - \hat y_{t-1} \|^2  | \MF_{t-1} \right] \\
    \label{eq:dd-2}
    &\leq -\phi_{t-1}(\hat y_{t-1}) + \frac{1}{2\lambda} \E \left[\| y_{t-1} - \eta_\sy g_{t-1} - \hat y_{t-1} \|^2 | \MF_{t-1} \right] \\
    & = -\phi_{t-1}(\hat y_{t-1}) +  \frac{1}{2\lambda} \| y_{t-1} - \hat y_{t-1} \|^2 + \frac{\eta_\sy^2}{2\lambda} \E \left[ \| g_{t-1} \|^2 | \MF_{t-1} \right] + \frac{\eta_\sy}{2\lambda} \E \left[ \lng \hat y_{t-1} - y_{t-1}, g_{t-1} \rng | \MF_{t-1}\right] \\
    \label{eq:dd-4}
    & \leq -(\phi_{t-1})_\lambda(y_{t-1}) + \frac{\eta_\sy}{2\lambda} \E \left[ \lng \hat y_{t-1} - y_{t-1}, g_{t-1} \rng | \MF_{t-1} \right] + \frac{\eta_\sy^2 (L^2 + \sigma_\sy^2)}{2\lambda}\\
    \label{eq:dd-5}
    & = -(\phi_{t-1})_\lambda(y_{t-1}) + \frac{\eta_\sy}{2\lambda} \lng \hat y_{t-1} - y_{t-1}, \grad_\sy f(x_{t-1}, y_{t-1}) \rng  + \frac{\eta_\sy^2 (L^2 + \sigma_\sy^2)}{2\lambda} \\
    \label{eq:dd-6}
    & \leq -(\phi_{t-1})_\lambda(y_{t-1}) + \frac{\eta_\sy}{2\lambda} \left(-\phi_{t-1}(\hat y_{t-1}) + \phi_{t-1}(y_{t-1}) + \frac{1}{2\lambda} \| y_{t-1} - \hat y_{t-1} \|^2 \right) + \frac{\eta_\sy^2 (L^2 + \sigma_\sy^2)}{2\lambda},
  \end{align}
  where \pref{eq:dd-1} is by the definition of the prox-mapping,
  \pref{eq:dd-2} is by the definition of projection onto a convex set,
  \pref{eq:dd-4} is by the definition of the Moreau envelope,
  \pref{eq:dd-5} holds because $g_{t-1}$ is an unbiased estimator of the gradient, and \pref{eq:dd-6} follows since $f$ is $\ell$-smooth and $\lambda \leq 1/\ell$. 
\end{proof}

\subsection{Proof of Theorem \ref*{thm:sgda}}
\begin{proof}[\pfref{thm:sgda}]
By the fact that $f$ satisfies \pref{asm:lip-smooth-f} and
\pref{lem:moreau-kl} on the function $y \mapsto \phi_t(y) =
-f(x_t, y)$, for any $\lambda\in(0,1/\ls)$ we have
  \begin{align}
    \Delta_t &\ldef{} f(x_t, y^*(x_t)) - f(x_t, y_t) \nonumber\\
             & \leq \frac{L \lambda + 1}{\mu_\sy} \cdot \underbrace{\| \grad \psi_{t,\lambda}(y_t) \|}_{\rdef\Gamma_t} + \frac{\gdy}{\mu_\sy}.\label{eq:bound-delta-gamma}
  \end{align}
We next observe that %
\begin{align*}
  & \E[\Phi_{1/2\ell}(x_{T+1})] \\
  & \leq \E[\Phi_{1/2\ell}(x_0)] + 2 \eta_\sx \ell \left( \sum_{t=0}^T \E[\Delta_t]\right) - \frac{\eta_\sx}{4} \left( \sum_{t=0}^T \E \left[\| \grad \Phi_{1/2\ell}(x_t) \|^2\right] \right) + \eta_\sx^2 \ell (L^2 + \sigma_\sx^2) (T+1)\\
  & \leq \E[\Phi_{1/2\ell}(x_0)] + 2 \eta_\sx \ell \left( \sum_{t=0}^T \E\left[ \frac{L\lambda+1}{\mu_\sy} \| \grad \psi_{t,\lambda}(y_t) \| + \frac{\gdy}{\mu_\sy} \right]\right) - \frac{\eta_\sx}{4} \left( \sum_{t=0}^T \E \left[\| \grad \Phi_{1/2\ell}(x_t) \|^2\right] \right) \\
  & ~~~~+ \eta_\sx^2 \ell (L^2 + \sigma_\sx^2) (T+1) \\
  & \leq \E[\Phi_{1/2\ell}(x_0)] + \frac{2 \eta_\sx \ell (L\lambda +1)}{\mu_\sy} \left( \sqrt{\frac{(D_\MX + D_\MY)L T}{\eta_\sy  (1 - \lambda \ell)}} + T  \sqrt{\frac{L \sqrt{L^2 + \sigma_\sx^2} \eta_\sx}{\eta_\sy (1 - \lambda \ell)}} + T \sqrt{\frac{(L^2 + \sigma_\sy^2)  \eta_\sy}{2\lambda(1 - \lambda \ell)}}\right)   \\
  & ~~~~+ \frac{2\eta_\sx \ell(T+1)\gdy}{\mu_\sy} - \frac{\eta_\sx}{4} \left( \sum_{t=0}^T \E \left[\| \grad \Phi_{1/2\ell}(x_t) \|^2\right] \right) + \eta_\sx^2 \ell (L^2 + \sigma_\sx^2) (T+1),
\end{align*}
where the first inequality follows from summing the guarantee of \pref{lem:lin-moreau} for $t
= 1, 2, \ldots, T+1$, the second inequality comes from
\pref{eq:bound-delta-gamma}, and the third inequality comes from
\pref{lem:bound-gammat}. Setting $\lambda = 1/2\ell$ and rearranging gives, for $\Delta_\Phi := \E[\Phi_{1/2\ell}(x_0) - \Phi_{1/2\ell}(x_{T+1})]$,
\begin{align*}
&\frac{1}{T+1}\sum_{t=0}^T \E \left[ \| \grad \Phi_{1/\ell}(x_t) \|^2
                 \right] \\
  &\leq \frac{4 \Delta_\Phi}{T\eta_\sx} + \frac{8\ell(L/\ell+1)}{\mu_\sy} \sqrt{\frac{2(D_\MX + D_\MY) L}{T \eta_\sy}} + \frac{8\ell(L/\ell+1)}{\mu_\sy} \sqrt{\frac{2L \sqrt{L^2 + \sigma_\sx^2}\eta_\sx}{\eta_\sy}} + \frac{8\ell(L/\ell+1)}{\mu_\sy} \sqrt{\ell (L^2 + \sigma_\sy^2) \eta_\sy} \\
  &~~~~+ 4 \eta_\sx \ell (L^2 + \sigma_\sx^2) + \frac{8\ell \gdy}{\mu_\sy}\\
  &\leq \frac{4 D_\MX L}{T\eta_\sx} + \frac{8\ell(L/\ell+1)}{\mu_\sy} \sqrt{\frac{2(D_\MX + D_\MY) L}{T \eta_\sy}} + \frac{8\ell(L/\ell+1)}{\mu_\sy} \sqrt{\frac{2L \sqrt{L^2 + \sigma_\sx^2}\eta_\sx}{\eta_\sy}} + \frac{8\ell(L/\ell+1)}{\mu_\sy} \sqrt{\ell (L^2 + \sigma_\sy^2) \eta_\sy} \\
  &~~~~+ 4 \eta_\sx \ell (L^2 + \sigma_\sx^2) + \frac{8\ell \gdy}{\mu_\sy}.
\end{align*}
Next, for a sufficiently large constant $C > 0$ and for any $\ep > 0$, set 
\begin{align*}
  \eta_\sy &\leq \frac{\ep^4\mu_\sy^2}{C \ell^3 (L^2 + \sigma_\sy^2)(L/\ell + 1)^2} \\
  \eta_\sx &\leq \frac{\ep^8\mu_\sy^4}{C\ell^5 L (L/\ell+1)^4(L^2 + \sigma_\sy^2) \sqrt{L^2 + \sigma_\sx^2}} \wedge \frac{\ep^2}{C\ell(L^2 + \sigma_\sx^2)}.
\end{align*}
Then as long as
\begin{equation}
  \label{eq:T-lb}
T \geq \frac{C(D_\MX + D_\MY) L}{\ep^2 \eta_\sx},
\end{equation}
as long as $C$ is sufficiently large, we get that
$$
\frac{1}{T+1} \sum_{t=0}^T \E \left[ \| \grad \Phi_{1/2\ell}(x_t) \| \right] \leq \ep.
$$
Here we have used that if $T$ is set as in \pref{eq:T-lb}, then
$$
\frac{\ell(L/\ell+1)}{\mu_\sy} \sqrt{\frac{(D_\MX + D_\MY)L}{T\eta_\sy}} \leq \frac{\ell(L/\ell+1)}{\mu_\sy} \cdot \frac{\ep^3 \mu_\sy}{\sqrt{C} \ell^{3/2} \sqrt{L^2 + \sigma_\sy^2} (L/\ell+1)} \leq \ep^2.
$$

Finally, the guarantee for function value suboptimality follows by
applying \pref{lem:kl-moreau}.

\end{proof}

\subsection{Supporting Lemmas}
Given a convex set $\MX \subset \BR^n$ and a point $x \in \MX$, the {\it normal cone} of $\MX$ at $x$ is the set
$$
N_\MX(x) := \{ x' \in \MX : \lng x', y- x \rng \leq 0 \ \forall y \in \MX \},
$$
and the {\it tangent cone } of $\MX$ at $x$ is the set
$$
T_\MX(x) := {\rm cl} \{ a \cdot (y - x) : y \in \MX , a \geq 0\},
$$
where ${\rm cl}$ denotes closure. It is well-known \citep{rockafellar1970convex} that for any $v \in N_\MX(x)$, for all $u \in T_\MX(x)$, we have that $\lng v, u \rng \leq 0$ (in other words, $N_\MX(x)$ is contained in the polar of $T_\MX(x)$; in fact, $N_\MX(x)$ is equal to the polar of $T_\MX(x)$).
\begin{lemma}
  \label{lem:moreau-kl}
  Suppose that $\phi : \MX \ra \BR$ is $\ell$-smooth, $L$-Lipschitz, and satisfies the gradient domination condition
  $$
  \max_{\bar x \in \MX, \| \bar x - x \| \leq 1} \lng x - \bar x, \grad \phi( x) \rng
  \geq \mu \cdot (\phi(x) - \phi(x^*)) - \gd,
  $$
  for some $\gd \geq 0, \mu > 0$.
Then for any $\lambda \in (0, 1/\ell)$, the Moreau envelope $\phi_\lambda(\cdot)$ satisfies
$$
\left\| \grad \phi_\lambda(x) \right\| \geq \frac{\mu}{L \lambda + 1} \cdot (\phi(x) - \phi(x^*)) - \frac{\gd}{L\lambda + 1}.
$$
\end{lemma}
\begin{proof}[\pfref{lem:moreau-kl}]
Fix $x \in \MX$. Let
\begin{equation}
  \label{eq:moreau-kl-prox}
\hat x := \prox_{\lambda \phi}(x) = \argmin_{x' \in \MX} \left\{ \phi(x') + \frac{1}{2\lambda} \| x - x' \|^2 \right\}.
\end{equation}
The first-order optimality conditions to \pref{eq:moreau-kl-prox} imply that
$$
 \grad \phi(\hat x) \in \frac{1}{\lambda} \cdot (\hat x - x)  + N_\MX(\hat x) \subseteq N_\MX(\hat x) + \frac{1}{\lambda} \| \hat x - x \| \cdot B_2(1),
$$
where $N_\MX(\hat x)$ denotes the normal cone of $\MX$ at $\hat x$. Since for any $\bar x \in \MX$, $\bar x - \hat x$ is in the tangent cone at $\hat x$, it follows that
$$
\mu \cdot (\phi(\hat x) - \phi(x^*)) \leq \max_{\bar x \in \MX, \| \bar x - \hat x \| \leq 1} \lng \hat x - \bar x, \grad \phi(\hat x) \rng + \gd \leq \frac{1}{\lambda} \cdot \| \hat x - x \| + \gd.
$$
Note that $\frac{1}{\lambda}(x - \hat x) = \grad \phi_\lambda(x)$. Thus, using also that $\phi$ is $L$-Lipschitz, we arrive at
\begin{align*}
  \mu \cdot (\phi(x) - \phi(x^*)) & \leq \mu \cdot (\phi(\hat x) - \phi(x^*)) + L \cdot \| \hat x - x \| \\
                                  & \leq \left( L + \frac{1}{\lambda} \right) \| \hat x - x \| + \gd \\
                                  & = \left( L \lambda + 1\right) \cdot \|\grad \phi_\lambda(x)\| + \gd.
\end{align*}
\end{proof}

The next lemma (\pref{lem:kl-moreau}) shows how to convert an
$\ep$-approximate stationary point with respect to the Moreau envelope into an approximate minimizer for functions $f$ satisfying \pref{asm:lip-smooth-f}.
\begin{lemma}
  \label{lem:kl-moreau}
  Suppose that $f$ satisfies the conditions of \pref{asm:lip-smooth-f}. Then for all $x \in \MX$,
  \begin{equation}
    \label{eq:kl-moreau}
\Phi(x) - \Phi(x^*) \leq \left( \frac{1}{\mu_\sx}+ \frac{L}{2\ell} \right) \cdot \| \nabla \Phi_{1/2\ell}(x) \| + \frac{\gdx}{\mu_\sx}.
\end{equation}
\end{lemma}

\begin{proof}[\pfref{lem:kl-moreau}]
  We first establish the statement of \pref{lem:kl-moreau} for points $x \in \MX$ for which $\Phi$ is differentiable at $x$. Suppose $x$ is such a point. Since a convex function is differentiable at a point if and only if its subgradient is a singleton at that point \cite[Theorem 25.1]{rockafellar1970convex}, it follows from \pref{eq:subgradient-wc} that $\partial \Phi(x)$ is a single vector, which we denote by $\grad \Phi(x)$.

  We first show that $\Phi(x)$ satisfies the following KL-type inequality (see also \cite[Lemma A.3]{yang2020global}, which shows a similar statement): 
  \begin{equation}
    \label{eq:nonsmooth-kl}
    \| \grad \Phi(x) \| \geq \mu_\sx \cdot (\Phi(x) - \Phi(x^*)) - \gdx.
\end{equation}
To prove \pref{eq:nonsmooth-kl}, fix any $y \in Y(x)$ (so that $f(x,y) = \Phi(x)$), and note that by item \ref{it:x-pl} of \pref{asm:lip-smooth-f}, we have that
$$
\max_{\bar x \in \MX, \| \bar x - x \| \leq 1} \lng x - \bar x, \grad_x f(x,y) \rng \geq \mu_\sx \cdot (f(x, y) - f(x^*(y), y)) - \gdx = \mu_\sx \cdot (\Phi(x) - f(x^*(y), y)) - \gdx.
$$
Note note that since $f(x', y) \leq \max_{y' \in \MY} f(x', y')$ for each $x'$, 
$$
f(x^*(y), y) = \min_{x' \in \MX} f(x', y) \leq \min_{x' \in \MX} \left[ \max_{y' \in \MY} f(x', y') \right] = \Phi(x^*).
$$
It follows that
\begin{equation}
  \label{eq:each-y-dom}
\max_{\bar x \in \MX, \| \bar x - x \| \leq 1} \lng x - \bar x, \grad_x f(x,y) \rng \geq \mu_\sx \cdot (\Phi(x) - f(x^*(y), y))  - \gdx\geq \mu_\sx \cdot (\Phi(x) - \Phi(x^*)) - \gdx.
\end{equation}
By Danskin's theorem (\pref{thm:danskin}) we have that $\{ \grad \Phi(x) \} = \partial \Phi(x) = \conv \{ \grad_\sx f(x,y') : y' \in Y(x) \}$, so $\grad_\sx f(x,y) = \grad \Phi(x)$. From \pref{eq:each-y-dom} and Cauchy-Schwarz it follows that
$$
\| \grad \Phi(x) \| \geq \max_{\bar x \in \MX, \| \bar x - x \| \leq 1} \lng x - \bar x, \grad \Phi(x) \rng = \max_{\bar x \in \MX, \| \bar x - x \| \leq 1} \lng x - \bar x, \grad_x f(x,y) \rng \geq \mu_\sx \cdot (\Phi(x) - \Phi(x^*)) - \gdx,
$$
establishing \pref{eq:nonsmooth-kl}.
  
  We proceed with the proof of \pref{eq:kl-moreau}. Let $\iota  = \| \grad \Phi_{1/2\ell}(x) \|$. By item \ref{it:phi-moreau} of \pref{lem:davis-moreau-2018}, there is some $\hat x \in \MX$ so that $\| \hat x - x \| \leq \iota/(2\ell)$ and $\inf_{v \in \partial \Phi(\hat x)} \| v \| \leq \iota$. By \pref{eq:nonsmooth-kl}, we have
  $$
\Phi(\hat x) - \Phi(x^*) \leq \frac{\iota + \gdx}{\mu_\sx}.%
$$
Item \ref{it:phi-lip} of \pref{lem:phi-facts} gives that $\Phi$ is $L$-Lipschitz, and hence
$$
\Phi(x) - \Phi(x^*) \leq \frac{\iota + \gdx}{\mu_\sx} + L \cdot \| \hat x - x \| \leq \iota \cdot \left( \frac{1}{\mu_\sx} + \frac{L}{2\ell} \right) + \frac{\gdx}{\mu_\sx}.
$$

Next we consider any point $x$ for which $\Phi$ is not differentiable at $x$. In the event that the interior $\MX^\circ$ is dense in $\MX$, we may apply \pref{eq:subgradient-wc} together with \cite[Theorem 25.5]{rockafellar1970convex} to conclude that the set of points at which $\Phi$ is differentiable is dense in $\MX^\circ$, and thus in $\MX$. Let $x_k \ra x$ be a convergent sequence of points approaching a point $x \in \MX$ at which $\Phi(\cdot)$ is differentiable. Then the above argument establishes that for each $k$,
$$
\Phi(x_k) - \Phi(x^*) \leq \left( \frac{1}{\mu_\sx} + \frac{L}{2\ell} \right) \cdot \| \grad \Phi_{1/2\ell}(x_k) \| + \frac{\gdx}{\mu_\sx}.
$$
By continuity of $\Phi$ (\pref{lem:phi-facts}, item \ref{it:phi-lip}) and of $\grad \Phi_{1/2\ell}$ (\pref{lem:phi-facts}, item \ref{it:moreau-smooth}), it follows that \pref{eq:kl-moreau} holds at the point $x$.

Finally, we consider the case that $\MX^\circ$ is not dense in $\MX$ (e.g., $\MX^\circ$ may be empty). In this case we consider the neighborhood $\MX_\delta \supset\supset \MX$ defined in \pref{rem:empty-interior}, which have dense interior. Using the conclusion of the previous paragraph with $\MX$ replaced by $\MX_\delta$ gives that for all $x \in \MX_\delta$,
$$
\Phi^\delta(x) - \Phi^\delta(x^*) \leq \left( \frac{1}{\mu_\sx - \delta} + \frac{L}{2\ell} \right) \cdot \| \grad \Phi_{1/2\ell}^\delta(x) \| + \frac{\gdx + \delta}{\mu_\sx - \delta},
$$
where $\Phi^\delta : \MX_\delta \ra \BR$ represents the best-response function $\Phi$ defined with respect to the domain $\MX_\delta$. %
Taking $\delta \downarrow 0$ and using continuity of $\Phi^\delta, \grad \Phi_{1/2\ell}^\delta$ in $\MX_\delta$ as well as continuity of $\grad \Phi_{1/2\ell}^\delta(\cdot)$ with respect to $\delta$ ensures that \pref{eq:kl-moreau} holds for any $x \in \MX$.
\end{proof}

\begin{lemma}
  \label{lem:bound-gammat}
 For the iterates of two-timescale SGDA, we have
  \begin{equation}
    \label{eq:sum-deltat-stoch}
\sum_{t=0}^{T-1} \E[\Gamma_t] \leq \sqrt{\frac{(D_\MX + D_\MY)L
    T}{\eta_\sy  (1 - \lambda \ell)}} + T  \sqrt{\frac{L \sqrt{L^2 +
      \sigma_\sx^2} \eta_\sx}{\eta_\sy (1 - \lambda \ell)}} + T
\sqrt{\frac{(L^2 + \sigma_\sy^2)  \eta_\sy}{2\lambda(1 - \lambda
    \ell)}}, 
\end{equation}
where we recall that $\Gamma_t\ldef{}\| \grad \psi_{t,\lambda}(y_t) \|$.
\end{lemma}
\begin{proof}[\pfref{lem:bound-gammat}]
Adding the inequality \pref{eq:deltat-dec-stoch} for $t = 1, 2, \ldots, T$ and using Jensen's inequality, we have
\begin{align*}
  \E[\psi_{T,\lambda}(y_T) - \psi_{0,\lambda}(y_0)] & \geq \sum_{t=1}^T \eta_\sy \lambda (1/\lambda - \ell) \E[\Gamma_{t-1}^2]- T \cdot \left( \eta_\sx L \sqrt{L^2 + \sigma_\sx^2} + \frac{\eta_\sy^2(L^2 + \sigma_\sy^2)}{2\lambda}\right)\\
                          & \geq \sum_{t=1}^T \eta_\sy \lambda (1/\lambda - \ell) \E[\Gamma_{t-1}]^2- T \cdot \left( \eta_\sx L \sqrt{L^2 + \sigma_\sx^2} + \frac{\eta_\sy^2(L^2 + \sigma_\sy^2)}{2\lambda}\right).
\end{align*}
For $\lambda \in (0,1/2\ell)$, we have $1/\lambda - \ell \geq \ell$. 
Noting that $| \E[\psi_{T,\lambda}(y_T) - \psi_{0,\lambda}(y_0)]| \leq (D_\MX + D_\MY) \cdot L$ since $f : \MX \times \MY \ra \BR$ is $L$-Lipschitz, it follows that
$$
\sqrt{\sum_{t=0}^{T-1} \E[\Gamma_t]^2} \leq \sqrt{\frac{(D_\MX + D_\MY) L +  T \cdot \left( \eta_\sx L \sqrt{L^2 + \sigma_\sx^2} + \eta_\sy^2 (L^2 + \sigma_\sy^2)/(2\lambda)\right)}{\eta_\sy (1 - \lambda \ell)}}.
$$
The conclusion \pref{eq:sum-deltat-stoch} follows by Cauchy-Schwarz and the inequality $\sqrt{x+y} \leq \sqrt x + \sqrt y$ for $x,y \geq 0$.

\end{proof}

%% file: appendix_last_iterate.tex
Below we prove \pref{prop:mvi_counterex}. Recall that $V(x,y)  = \frac{\lng x, Ry \rng}{\lng x, Sy \rng}$ with $R,S$ given by \pref{eq:AS}, and for $z = (x,y)$, $F(z) = (\grad_x V(x,y) - \grad_y V(x,y))$. 
\begin{proof}[\pfref{prop:mvi_counterex}]
  We first verify that $z^{\star}$ is the unique Nash equilibrium. Note that
  \begin{align*}
    \Phi(x) =& \max_y V(x,y) =   \max \left\{ \frac{-x_1 - \ep x_2}{sx_1 + x_2}, \frac{\ep x_1}{sx_1 + x_2} \right\} = \frac{\ep x_1}{sx_1 + x_2} > 0 \text{\ \ for $x_1 > 0$}\\
    \Psi(y) =& \min_x V(x,y) =  \min \left\{ \frac{-y_1 + \ep y_2}{sy_1 + sy_2}, \frac{-\ep y_1}{y_1 + y_2} \right\} < 0 \text{\ \ for $y_1 > 0$}.
  \end{align*}
  The unique global minimum of $\Phi(\cdot)$ over $\MX = \Delta(\cA)$ is at $(x_1, x_2) = (0,1)$, and the unique global maximum of $\Psi(\cdot)$ over $\MY = \Delta(\cB)$ is at $(y_1, y_2) = (0,1)$. This verifies that $z^{\star}$ is the unique global Nash equilibrium. The value of the game is $V(x^{\star}, y^{\star}) = 0$. 

  Now consider the point $z = (x,y)$, where $x = y = (1,0)$. Then
  \begin{align*}
    & \lng F(z), z - z^{\star} \rng \\
    & = \frac{1}{(x^\t Sy)^2} \cdot \left[ \lng (x^\t Sy) \cdot Ry - (x^\t Ry) \cdot Sy, x - x^{\star} \rng + \lng - (x^\t Sy) \cdot (R^\t x) + (x^\t Ry) \cdot (S^\t x), y - y^{\star} \rng \right] \\
    &= \frac{1}{s^2} \cdot \left[ -(x^\t Sy) ((x^{\star})^\t Ry) + (x^\t Ry) \cdot ((x^{\star})^\t Sy) + (x^\t Sy) \cdot (x^\t Ry^{\star}) - (x^\t Ry) \cdot (x^\t S y^{\star}) \right] \\
    & = \frac{1}{s^2} \cdot \left[ (x^\t Ry) \cdot ((x^{\star})^\t Sy - x^\t Sy^{\star}) + (x^\t Sy) \cdot (x^\t Ry^{\star} - (x^{\star})^\t Ry) \right] \\
    &= \frac{1}{s^2} \cdot \left[ -1 \cdot (1-s) + s \cdot (\ep - (-\ep)\right] \\
    &= \frac{1}{s^2} \cdot \left(s + 2\ep s - 1\right),
  \end{align*}
  which is negative for sufficiently small $\ep$ (in particular, for $\ep < \frac{1-s}{2s}$). %

  Finally, we check that the MVI property
  \begin{equation}
    \label{eq:mvi-proof-prop}
\lng F(z), z - \hat z \rng \geq 0 \quad \forall z \in \MZ
\end{equation}
  fails for all $\hat z = (\hat x, \hat y)$ which are not a Nash equilibrium. For any $\hat z$ which is not a Nash equilibrium, either the \minplayer or \maxplayer can deviate from their policy in a way that increases their utility; we assume without loss it is the \minplayer (the case for the \maxplayer is symmetric). In particular, there is some $x \in \MX$ so that
  $$
  V(x, \hat y) = \frac{x^\t R \hat y}{x^\t S \hat y} < \frac{\hat x^\t R \hat y}{\hat x^\t S \hat y} = V(\hat x, \hat y).
  $$
  It follows that
  \begin{align*}
    0 &> \lng x, (\hat x^\t S \hat y) \cdot R \hat y - (\hat x^\t R \hat y) \cdot S \hat y \rng \\
      &= \lng x - \hat x, (\hat x^\t S \hat y) \cdot R \hat y - (\hat x^\t R \hat y) \cdot S \hat y \rng \\
      &=  (\hat x^\t S \hat y)^2 \cdot \lng x - \hat x,\grad_x V(\hat x, \hat y) \rng.
  \end{align*}
  It follows that $\lng x - \hat x, \grad_x V(\hat x, \hat y) \rng < 0$. For $\alpha \in [0,1]$, define $x_\alpha = (1-\alpha) \hat x + \alpha x$. By continuity of the function $\alpha \mapsto \grad_x V(x_\alpha, \hat y)$, there must be some $\alpha \in (0,1)$ so that
  $$
\frac{1}{\alpha} \lng x_\alpha - \hat x, \grad_x V(x_\alpha, \hat y) \rng =  \lng x - \hat x, \grad_x V(x_\alpha, \hat y) \rng < 0.
  $$

  Letting $z := (x_\alpha, \hat y)$, we obtain that $\lng z - \hat z, F(z)\rng < 0$, violating \pref{eq:mvi-proof-prop}.
\end{proof}
  We remark that an alternative way to  verify that $(x^{\star}, y^{\star})$ is a Nash equilibrium in the above proof is as follows: we may calculate that
  \begin{align*}
    \nabla_x V(x^{\star}, y^{\star}) =&  Ry^{\star} = \matx{\ep \\ 0},\\
    \nabla_y V(x^{\star}, y^{\star}) =& R^\t x^{\star} = \matx{-\ep \\ 0},
  \end{align*}
  which shows that $x^{\star}$ satisfies the first-order optimality conditions for minimizing $x \mapsto V(x, y^{\star})$, and $y^{\star}$ satisfies the first-order optimality conditions for maximizing $y \mapsto V(x^{\star}, y)$. It is then straightforward to check that in fact $x^{\star}$ is a global minimizer of $x \mapsto V(x,y^{\star})$, and $y^{\star}$ is a global minimizer of $y \mapsto V(x^{\star}, y)$.

  \subsection{Experimental Details}
  \label{app:experimental-details}
\pref{fig:mvi_easy} and \pref{fig:convergence_easy} use the following
game,  which is the game from \pref{prop:mvi_counterex} with $\ep = 0.1, s = 0.3$:
$$
R = \matx{-1.0 & 0.1 \\ -0.1 & 0.0} , \qquad S = \matx{0.3 & 0.3 \\ 1.0 & 1.0}.
$$

\pref{fig:mvi_hard} and \pref{fig:convergence_hard} use the following
game, which is a rounded version of a game we found via a random search:
$$
R = \matx{-0.6 & -0.3 \\ 0.6 & -0.3} , \qquad S = \matx{0.9 & 0.5 \\ 0.8 & 0.4}.
$$